\documentclass[english]{article}
\usepackage[T1]{fontenc}
\usepackage[latin9]{inputenc}
\usepackage{geometry}
\usepackage{color}
\usepackage{float}
\usepackage{amsmath}
\usepackage{graphicx}
\usepackage{esint}
\usepackage[numbers]{natbib}

\makeatletter

\floatstyle{ruled}
\newfloat{algorithm}{tbp}{loa}
\providecommand{\algorithmname}{Algorithm}
\floatname{algorithm}{\protect\algorithmname}

\usepackage{nips13submit_e,times}
\usepackage{hyperref}
\usepackage{url}
\nipsfinaltrue

\title{Learning Gaussian Graphical Models with Observed or Latent FVSs} 

\author{ {\bf Ying Liu} \\  
Department of EECS \\  
Massachusetts Institute of Technology
\And 
{\bf Alan S. Willsky}  \\ 
Department of EECS\\ 
Massachusetts Institute of Technology              
} 

 \usepackage{amsthm}
\newtheorem{proposition}{Proposition}
\newtheorem{lemma}{Lemma}
\usepackage{natbib}
\@ifundefined{showcaptionsetup}{}{%
 \PassOptionsToPackage{caption=false}{subfig}}
\usepackage{subfig}
\makeatother

\usepackage{babel}
\begin{document}
\global\long\def\calN{\mathcal{N}}

\global\long\def\calE{\mathcal{E}}
\global\long\def\calV{\mathcal{V}}
 \global\long\def\calT{\mathcal{T}}
 \global\long\def\calF{\mathcal{F}}
 \global\long\def\calB{\mathcal{B}}
\global\long\def\calP{\mathcal{P}}
 \global\long\def\pF{\widetilde{\mathcal{F}}}
 \global\long\def\pT{\widetilde{\mathcal{T}}}
 \global\long\def\calO{\mathcal{O}}
 \global\long\def\calG{\mathcal{G}}
\global\long\def\calQ{\mathcal{Q}}
\global\long\def\calC{\mathcal{C}}
 \global\long\def\hatP{\hat{P}}
\global\long\def\hatI{\hat{I}}
 \global\long\def\hatp{\hat{p}}
 \global\long\def\rx{\mathbf{\mathrm{x}}}
\global\long\def\bx{\mathbf{x}}
\global\long\def\bX{\mathbf{\mathrm{X}}}
 \global\long\def\ry{\mathbf{\mathrm{y}}}
 \global\long\def\by{\mathbf{y}}
 \global\long\def\bz{\mathbf{z}}
 \global\long\def\bh{\mathbf{h}}
 \global\long\def\bg{\mathbf{g}}
 \global\long\def\be{\mathbf{e}}
 \global\long\def\bzero{\mathbf{0}}
 \global\long\def\bone{\mathbf{1}}
 \global\long\def\bmu{\boldsymbol{\mu}}
 \global\long\def\btheta{\boldsymbol{\theta}}
 \global\long\def\walksum#1#2#3{\phi(#1\stackrel{#3}{\rightarrow}#2)}
 \global\long\def\walksumlong#1#2#3{\phi(#1\stackrel{#3}{\longrightarrow}#2)}
 \global\long\def\walksumloong#1#2#3{\phi(#1\stackrel{#3}{\xrightarrow{\hspace{1em}*{1.5cm}}}#2)}
 \global\long\def\singlewalksum#1#2#3{\phi^{(1)}(#1\stackrel{#3}{\rightarrow}#2)}
 \global\long\def\mes#1#2#3{\Delta#1_{#2\rightarrow#3}}
 \global\long\def\deff{\stackrel{\Delta}{=}}
 \global\long\def\expec{\mathbb{}}

\author{ {\bf Ying Liu} \\   Department of EECS \\   Massachusetts Institute of Technology \\\texttt{liu\_ying@mit.edu} \And  {\bf Alan S. Willsky}  \\  Department of EECS\\  Massachusetts Institute of Technology\\\texttt{willsky@mit.edu}}\vspace{-0.1in}               

\maketitle
\begin{abstract}
Gaussian Graphical Models (GGMs) or Gauss Markov random fields are
widely used in many applications, and the trade-off between the modeling
capacity and the efficiency of learning and inference has been an
important research problem. In this paper, we study the family of
GGMs with small feedback vertex sets (FVSs), where an FVS is a set
of nodes whose removal breaks all the cycles. Exact inference such
as computing the marginal distributions and the partition function
has complexity ${\cal O}(k^{2}n)$ using message-passing algorithms,
where $k$ is the size of the FVS, and $n$ is the total number of
nodes. We propose efficient structure learning algorithms for two
cases: 1) All nodes are observed, which is useful in modeling social
or flight networks where the FVS nodes often correspond to a small
number of highly influential nodes, or hubs, while the rest of the
networks is modeled by a tree. Regardless of the maximum degree, without
knowing the full graph structure, we can \textit{exactly} compute
the maximum likelihood estimate with complexity ${\cal O}(kn^{2}+n^{2}\log n)$
if the FVS is known or in polynomial time if the FVS is unknown but
has bounded size. 2) The FVS nodes are latent variables, where structure
learning is equivalent to decomposing \textcolor{black}{an }inverse
covariance matrix (exactly or approximately) into the sum of a tree-structured
matrix and a low-rank matrix. By incorporating efficient inference
into the learning steps, we can obtain a learning algorithm using
alternating low-rank corrections with complexity $\calO(kn^{2}+n^{2}\log n)$
per iteration. We perform experiments using both synthetic data as
well as real data of flight delays to demonstrate the modeling capacity
with FVSs of various sizes. 

\setlength{\parsep}{-5mm}

\end{abstract}

\section{Introduction}

\vspace{-0.1in}

In undirected graphical models or Markov random fields, each node
represents a random variable while the set of edges specifies the
conditional independencies of the underlying distribution. When the
random variables are jointly Gaussian, the models are called \textit{Gaussian
graphical models }(GGMs) or \textit{Gauss Markov random fields}. GGMs,
such as linear state space models, Bayesian linear regression models,
and thin-membrane/thin-plate models, have been widely used in communication,
image processing, medical diagnostics, and gene regulatory networks.
In general, a larger family of graphs represent a larger collection
of distributions and thus can better approximate arbitrary empirical
distributions. However, many graphs lead to computationally expensive
inference and learning algorithms. Hence, it is important to study
the trade-off between modeling capacity and efficiency. 

Both inference and learning are efficient for tree-structured graphs
(graphs without cycles): inference can be computed exactly in linear
time (with respect to the size of the graph) using belief propagation
(BP) \citep{pearl1986constraint} while the learning problem can be
solved exactly in quadratic time using the Chow-Liu algorithm \citep{chow1968approximating}.
Since trees have limited modeling capacity, many models beyond trees
have been proposed \citep{choi2009exploiting,comer1999segmentation,bouman1994multiscale,karger2001learning}.
Thin junction trees (graphs with low tree-width) are extensions of
trees, where inference can be solved efficiently using the junction
algorithm \citep{jordan2004graphical}. However, learning junction
trees with tree-width greater than one is NP-complete \citep{karger2001learning}
and tractable learning algorithms (e.g. \citep{abbeel2006learning})
often have constraints on both the tree-width and the maximum degree.
 Since graphs with large-degree nodes are important in modeling applications
such as social networks, flight networks, and robotic localization,
 we are interested in finding a family of models that allow arbitrarily
large degrees while being tractable for learning. 

Beyond thin-junction trees, the family of sparse GGMs is also widely
studied \citep{dobra2004sparse,tipping2001sparse}. These models are
often estimated using methods such as graphical lasso (or $l$-1 regularization)
\citep{friedman2008sparse,ravikumar2008model}. However, a sparse
GGM (e.g. a grid) does not automatically lead to efficient algorithms
for exact inference.  Hence, we are interested in finding a family
of models that are not only sparse but also have guaranteed efficient
inference algorithms. 

In this paper, we study the family of GGMs with small feedback vertex
sets (FVSs), where an FVS is a set of nodes whose removal breaks all
cycles \citep{vazirani2004approximation}. The authors of \citep{liu2012feedback}
have demonstrated that the computation of exact means and variances
for such a GGM can be accomplished, using message-passing algorithms
with complexity $\calO(k^{2}n)$, where $k$ is the size of the FVS
and $n$ is the total number of nodes. They have also presented results
showing that for models with larger FVSs, approximate inference (obtained
by replacing a full FVS by a pseudo-FVS) can work very well, with
empirical evidence indicating that a pseudo-FVS of size $\calO(\log n)$
gives excellent results. In Appendix \ref{sec:AppenPartitionFunction}
we will provide some additional analysis of inference for such models
(including the computation of the partition function), but the main
focus is maximum likelihood (ML) \textit{learning} of models with
FVSs of modest size, including identifying the nodes to include in
the FVS. 

In particular, we investigate two cases. In the first, all of the
variables, including any to be included in the FVS are observed. We
provide an algorithm for exact ML estimation that, regardless of the
maximum degree, has complexity $\calO(kn^{2}+n^{2}\log n)$ if the
FVS nodes are identified in advance and polynomial complexity if the
FVS is to be learned and of bounded size. Moreover, we provide an
approximate and much faster greedy algorithm when the FVS is unknown
\textit{and} large. In the second case, the FVS nodes are taken to
be latent variables. In this case, the structure learning problem
corresponds to the (exact or approximate) decomposition of an inverse
covariance matrix into the sum of a tree-structured matrix and a low-rank
matrix. We propose an algorithm that iterates between two projections,
which can also be interpreted as alternating \textit{low-rank} corrections.
We prove that even though the second projection is onto a highly non-convex
set, it is carried out exactly, thanks to the properties of GGMs of
this family. By carefully incorporating efficient inference into the
learning steps, we can further reduce the complexity to $\calO(kn^{2}+n^{2}\log n)$
per iteration.  We also perform experiments using both synthetic
data and real data of flight delays to demonstrate the modeling capacity
with FVSs of various sizes. We show that empirically the family of
GGMs of size $\calO(\log n)$ strikes a good balance between the modeling
capacity and efficiency. 

\vspace{-0.15in}

\paragraph*{Related Work}

In the context of classification, the authors of \citep{friedman1997bayesian}
have proposed the tree augmented naive Bayesian model, where the class
label variable itself can be viewed as a size-one observed FVS; however,
this model does not naturally extend to include a larger FVS. In \citep{chandrasekaran2010latent},
a convex optimization framework is proposed to learn GGMs with latent
variables, where conditioned on a small number of latent variables,
the remaining nodes induce a sparse graph. In our setting with latent
FVSs, we further require the sparse subgraph to have tree structure.

\vspace{-0.2in}

\section{Preliminaries}

\vspace{-0.2in}

Each undirected graphical model has an underlying graph $\calG=(\calV,\calE)$,
where $\calV$ denotes the set of vertices (nodes) and $\calE$ the
set of edges. Each node $s\in\calV$ corresponds to a random variable
$x_{s}$. When the random vector $\bx_{\calV}$ is jointly Gaussian,
the model is a GGM with density function given by $p(\mathbf{x})=\frac{1}{Z}\exp\{-\frac{1}{2}\mathbf{x}^{T}J\mathbf{x}+\bh^{T}\mathbf{x}\}$,
where $J$ is the \textit{information matrix} or\textit{ precision
matrix}, $\bh$ is the \textit{potential vector}, and $Z$ is the
\textit{partition function}. The parameters $J$ and $\bh$ are related
to the mean $\boldsymbol{\mu}$ and covariance matrix $\Sigma$ by
$\boldsymbol{\mu}=J^{-1}\bh$ and $\Sigma=J^{-1}$. The structure
of the underlying graph is revealed by the sparsity pattern of $J$:
there is an edge between $i$ and $j$ if and only if $J_{ij}\neq0$.

Given samples $\{\bx^{i}\}_{i=1}^{s}$ independently generated from
an unknown distribution $q$ in the family $\calQ$, the ML estimate
is defined as $q_{\text{ML}}=\arg\min_{q\in\calQ}\sum_{i=1}^{s}\log q(\bx^{i}).$
For Gaussian distributions, the empirical distribution is $\hat{p}(\bx)=\calN(\bx;\hat{\bmu},\hat{\Sigma})$,
where the empirical mean $\hat{\bmu}=\frac{1}{s}\sum_{i=1}^{s}\bx^{i}$
and the empirical covariance matrix $\hat{\Sigma}=\frac{1}{s}\sum_{i=1}^{s}\bx^{i}\left(\bx^{i}\right)^{T}-\hat{\bmu}\hat{\bmu}^{T}$.
The Kullback-Leibler (K-L) divergence between two distributions $p$
and $q$ is defined as $D_{\text{KL}}(p||q)=\int p(\bx)\log\frac{p(\bx)}{q(\bx)}\mathrm{d}\bx$.
Without loss of generality, we assume in this paper the means are
zero. 

Tree-structured models are models whose underlying graphs do not have
cycles. The ML estimate of a tree-structured model can be computed
exactly using the Chow-Liu algorithm \citep{chow1968approximating}.
We use $\Sigma_{\text{CL}}=\text{CL}(\hat{\Sigma})$ and $\calE_{\text{CL}}=\text{CL}_{\calE}(\hat{\Sigma})$
to denote respectively the covariance matrix and the set of edges
learned using the Chow-Liu algorithm where the samples have empirical
covariance matrix $\hat{\Sigma}$.

\vspace{-0.15in}

\section{Gaussian Graphical Models with Known FVSs\label{sec:Gaussian-Graphical-Models}}

\vspace{-0.15in}

\noindent \label{subsec:FVS}In this section we briefly discuss some
of the ideas related to GGMs with FVSs of size $k$, where we will
also refer to the nodes in the FVS as \textit{feedback nodes}. An
example of a graph and its FVS is given in Figure \ref{fig:FVS},
where the full graph (Figure \ref{fig:FVSfullgraph}) becomes a cycle-free
graph (Figure \ref{fig:FVStreepart}) if nodes 1 and 2 are removed,
and thus the set $\{1,2\}$ is an FVS. \vspace{-0.2in}
\begin{figure}[H]
\centering{}%
\begin{minipage}[t]{0.7\columnwidth}%
\subfloat[]{\includegraphics[width=0.35\columnwidth]{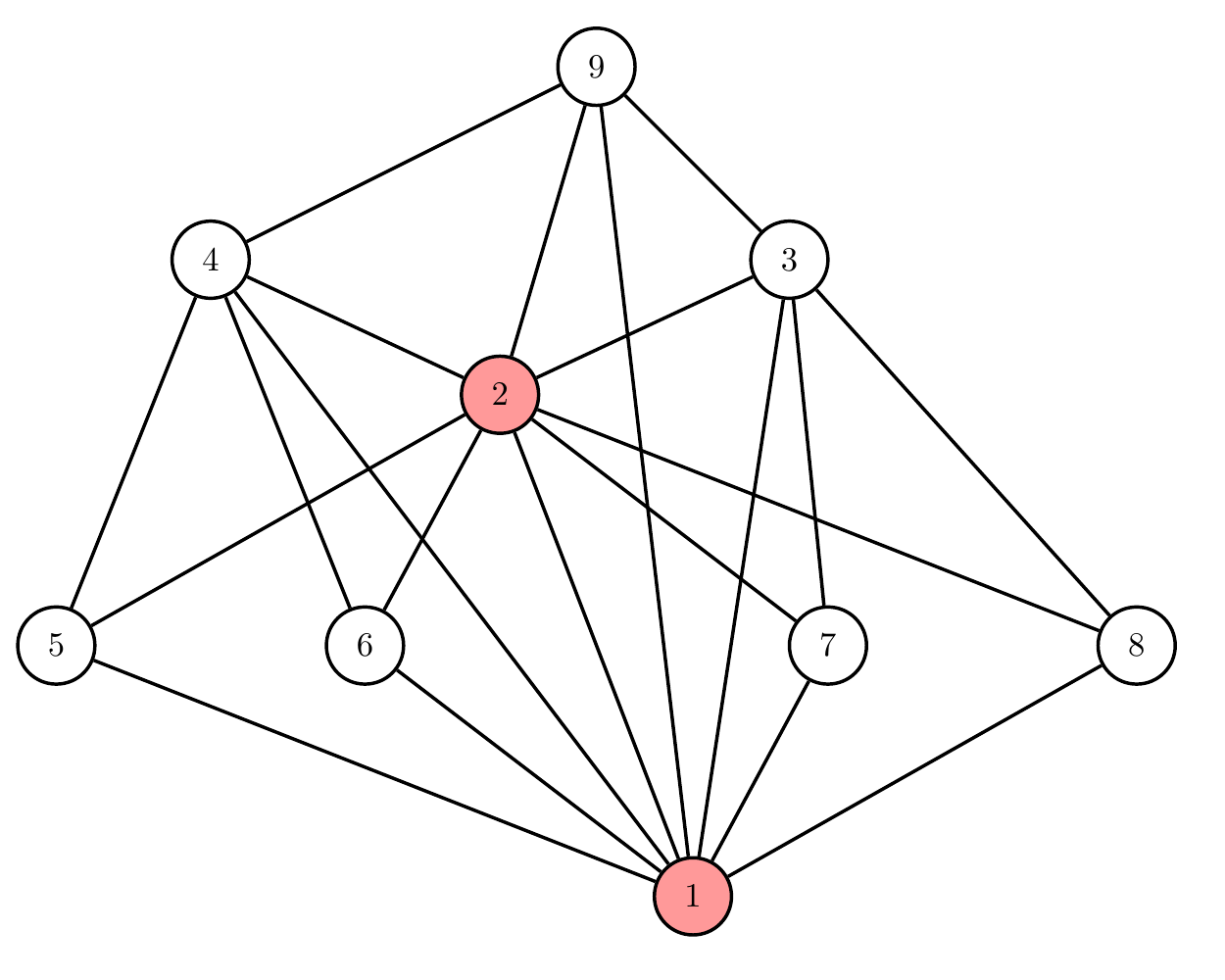}\label{fig:FVSfullgraph}

}\hfill{}\subfloat[]{\includegraphics[bb=0bp -50bp 360bp 216bp,width=0.35\columnwidth]{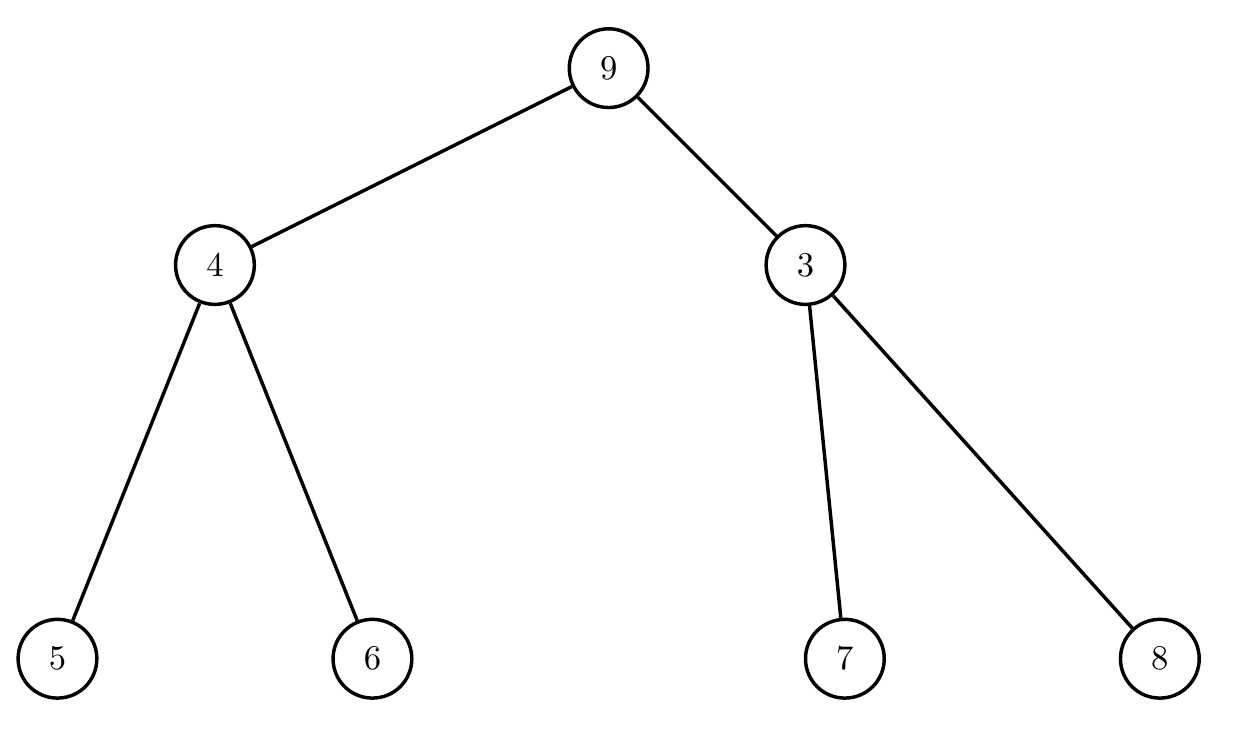}\label{fig:FVStreepart}

}\vspace{-0.1in}
\caption{A graph with an FVS of size 2. (a) Full graph; (b) Tree-structured
subgraph after removing nodes 1 and 2}

\label{fig:FVS}%
\end{minipage}
\end{figure}
\vspace{-0.2in}

\noindent Graphs with small FVSs have been studied in various contexts.
The authors of \citep{dinneen2001forbidden} have characterized the
family of graphs with small FVSs and their obstruction sets (sets
of forbidden minors). FVSs are also related to the ``stable sets''
in the study of tournaments \citep{brandt2011minimal}.  

Given a GGM with an FVS of size $k$ (where the FVS may or may not
be given), the marginal means and variances $\bmu_{i}=\left(J^{-1}\bh\right)_{i}$
and $\Sigma_{ii}=\left(J^{-1}\right)_{ii}$, for $\forall i\in\calV$
can be computed \textit{exactly} with complexity $\calO(k^{2}n)$
using the feedback message passing (FMP) algorithm proposed in \citep{liu2012feedback},
where standard BP is employed two times on the cycle-free subgraph
among the non-feedback nodes while a special message-passing protocol
is used for the FVS nodes. We provide a new algorithm in Appendix
\ref{sec:acceleratedLatent}, to compute $\det J$, the determinant
of $J$, and hence the partition function of such a model with complexity
$\calO(k^{2}n)$. The algorithm is described and proved in Appendix
\ref{sec:AppenPartitionFunction}. 

An important point to note is that the complexity of these algorithms
depends simply on the size $k$ and the number of nodes $n$. There
is no loss in generality in assuming that the size-$k$ FVS $F$ is
fully connected and each of the feedback nodes has edges to every
non-feedback node. In particular, after re-ordering the nodes so that
the elements of $F$ are the first $k$ nodes ($T=V\backslash F$
is the set of non-feedback nodes of size $n-k$), we have that $J=\left[\begin{array}{cc}
J_{F} & J_{M}^{T}\\
J_{M} & J_{T}
\end{array}\right]\succ0,$ where $J_{T}\succ0$ corresponds to a tree-structured subgraph among
the non-feedback nodes, $J_{F}\succ0$ corresponds to a complete graph
among the feedback nodes, and all entries of $J_{M}$ may be non-zero
as long as $J_{T}-J_{M}J_{F}^{-1}J_{M}^{T}\succ0$ (while $\Sigma=\left[\begin{array}{cc}
\Sigma_{F} & \Sigma{}_{M}^{T}\\
\Sigma_{M} & J_{T}
\end{array}\right]=J^{-1}\succ0$ ). We will refer to the family of such models with a given FVS $F$
as $\calQ_{F}$, and the class of models with some FVS of size at
most $k$ as $\calQ_{k}$.%
\footnote{In general a graph does not have a unique FVS. The family of graphs
with FVSs of size $k$ includes all graphs where there \textit{exists
}an FVS of size $k$.%
} If we are not explicitly given an FVS, though the problem of finding
an FVS of minimal size is NP-complete, the authors of \citep{bafna1992}
have proposed an efficient algorithm with complexity $\calO(\min\{m\log n,\ n^{2}\})$,
where $m$ is the number of edges, that yields an FVS at most twice
the minimum size (thus the inference complexity is increased only
by a constant factor). However, the main focus of this paper, explored
in the next section, is on \textit{learning }models with small FVSs
(so that when learned, the FVS is \textit{known}). As we will see,
the complexity of such algorithms is manageable. Moreover, as our
experiments will demonstrate, for many problems, quite modestly sized
FVSs suffice. \vspace{-0.15in}

\section{Learning GGMs with Observed or Latent FVS of Size $k$ \vspace{-0.15in}
}

In this section, we study the problem of recovering a GGM from $i.i.d.$
samples, where the feedback nodes are either observed or latent variables.
If all nodes are observed, the empirical distribution $\hat{p}(\bx_{F},\bx_{T})$
is parametrized by the empirical covariance matrix $\hat{\Sigma}=\left[\begin{array}{cc}
\hat{\Sigma}_{F} & \hat{\Sigma}_{M}^{T}\\
\hat{\Sigma}_{M} & \hat{\Sigma}_{T}
\end{array}\right]$. If the feedback nodes are latent variables, the empirical distribution
$\hat{p}(\bx_{T})$ has empirical covariance matrix $\hat{\Sigma}_{T}$.
With a slight abuse of notation, for a set $A\subset\calV$, we use
$q(\bx_{A})$ to denote the marginal distribution of $\bx_{A}$ under
a distribution $q(\bx_{\calV})$. \vspace{-0.1in}

\subsection{When All Nodes Are Observed}

\vspace{-0.1in}

When all nodes are observed, we have two cases: 1) When an FVS of
size $k$ is given, we propose the \textit{conditioned Chow-Liu algorithm,}
which computes the \textit{exact }ML estimate efficiently; 2) When
no FVS is given \textit{a priori}, we propose both an exact algorithm
and a greedy approximate algorithm for computing the ML estimate.\vspace{-0.1in}

\subsubsection{Case 1: An FVS of Size $k$ Is Given.}

\vspace{-0.1in}

\label{sub:knownFVS}

When a size-$k$ FVS $F$ is given, the learning problem becomes solving
\vspace{-0.15in}

\begin{alignat}{1}
p_{\text{ML}}(\bx_{F},\bx_{T}) & =\underset{q(\bx_{F},\bx_{T})\in\calQ_{F}}{\arg\min}D_{\text{KL}}(\hat{p}(\bx_{F},\bx_{T})||q(\bx_{F},\bx_{T})).\label{eq:knownFVS}
\end{alignat}

\vspace{-0.1in}
This optimization problem is defined on a highly non-convex set $\calQ_{F}$
with combinatorial structures: indeed, there are $(n-k)^{n-k-2}$
possible spanning trees among the subgraph induced by the non-feedback
nodes. However, we are able to solve Problem \eqref{eq:knownFVS}
\textit{exactly} using the conditioned Chow-Liu algorithm described
in Algorithm \ref{Algo:givenF}.%
\footnote{Note that the conditioned Chow-Liu algorithm here is different from
other variations of the Chow-Liu algorithm such as in \citep{kirshner2004conditional}
where the extensions are to enforce the inclusion or exclusion of
a set of edges. %
}\textit{ }The intuition behind this algorithm is that even though
the entire graph is not tree, the subgraph induced by the non-feedback
nodes (which corresponds to the distribution of the non-feedback nodes
conditioned on the feedback nodes) has tree structure, and thus we
can find the best tree among the non-feedback nodes using the Chow-Liu
algorithm applied on the conditional distribution. To obtain a concise
expression, we also exploit a property of Gaussian distributions:
the conditional information matrix (the information matrix of the
conditional distribution) is simply a submatrix of the whole information
matrix. In Step 1 of Algorithm \ref{Algo:givenF}, we compute the
conditional covariance matrix using the Schur complement, and then
in Step 2 we use the Chow-Liu algorithm to obtain the best approximate
$\Sigma_{\text{CL}}$ (whose inverse is tree-structured). In Step
3, we match exactly the covariance matrix among the feedback nodes
and the covariance matrix between the feedback nodes and the non-feedback
nodes. For the covariance matrix among the non-feedback nodes, we
add the matrix subtracted in Step 1 back to $\Sigma_{\text{CL}}$.
Proposition \ref{prop:givenF} states the correctness and the complexity
of Algorithm \ref{Algo:givenF}. Its proof included in Appendix \ref{sec:Appen_Prop1}.We
denote the output covariance matrix of this algorithm as $\text{CCL}(\hat{\Sigma})$.\vspace{-0.1in}

\begin{algorithm}[H]
\centering{}\textbf{}%
\begin{minipage}[t]{0.9\columnwidth}%
\textbf{Input:} $\hat{\Sigma}\succ0$ and an FVS $F$

\textbf{Output:} $\calE_{\text{ML}}$ and $\Sigma_{\text{ML}}$\textbf{ }
\begin{enumerate}
\item Compute the conditional covariance matrix $\hat{\Sigma}_{T|F}=\hat{\Sigma}_{T}-\hat{\Sigma}_{M}\hat{\Sigma}_{F}^{-1}\hat{\Sigma}_{M}^{T}$
.
\item Let $\Sigma_{\text{CL}}=\text{CL}(\hat{\Sigma}_{T|F})$ and $\calE_{\text{CL}}=\text{CL}_{\calE}(\hat{\Sigma}_{T|F})$. 
\item $\calE_{\text{ML}}=\calE_{\text{\text{CL}}}$ and $\Sigma_{\text{ML}}=\left[\begin{array}{cc}
\hat{\Sigma}_{F} & \hat{\Sigma}_{M}^{T}\\
\hat{\Sigma}_{M} & \Sigma_{\text{CL}}+\hat{\Sigma}_{M}\hat{\Sigma}_{F}^{-1}\hat{\Sigma}_{M}^{T}
\end{array}\right]$.
\end{enumerate}
\caption{The conditioned Chow-Liu algorithm}
\label{Algo:givenF}%
\end{minipage}
\end{algorithm}

\begin{proposition}

Algorithm \ref{Algo:givenF} computes the ML estimate $\Sigma_{\text{ML}}$
and $\calE_{\text{ML}}$, exactly with complexity $\calO(kn^{2}+n^{2}\log n)$.
In addition, all the non-zero entries of $J_{\text{ML}}\stackrel{\Delta}{=}\Sigma_{\text{ML}}^{-1}$
can be computed with extra complexity $\calO(k^{2}n)$.

\label{prop:givenF}

\end{proposition}

\vspace{-0.1in}

\subsubsection{Case 2: The FVS Is to Be Learned\vspace{-0.1in}
}

Structure learning becomes more computationally involved when the
FVS is unknown. In this subsection, we present both exact and approximate
algorithms for learning models with FVS of size no larger than $k$
(i.e., in $\calQ_{k}$). For a fixed empirical distribution $\hat{p}(\bx_{F},\bx_{T})$,
we define $d(F)$, a set function of the FVS $F$ as the minimum value
of \eqref{eq:knownFVS}, i.e.,\vspace{-0.15in}

\begin{equation}
d(F)=\min_{q(\bx_{F},\bx_{T})\in\calQ_{F}}D_{\text{KL}}(\hat{p}(\bx_{F},\bx_{T})||q(\bx_{F},\bx_{T})).\label{eq:setfunction}
\end{equation}
\vspace{-0.1in}

When the FVS is unknown, the ML estimate can be computed exactly by
enumerating all possible $\dbinom{n}{k}$ FVSs of size $k$ to find
the $F$ that minimizes $d(F)$. Hence, the exact solution can be
obtained with complexity $\calO(n^{k+2}k)$, which is polynomial in
$n$ for fixed $k$. However, as our empirical results suggest, choosing
$k=\calO(\log(n))$ works well, leading to quasi-polynomial complexity
even for this exact algorithm. That observation notwithstanding, the
following greedy algorithm (Algorithm \ref{algo:greedynew}), which,
at each iteration, selects the single best node to add to the current
set of feedback nodes, has polynomial complexity for arbitrarily large
FVSs. As we will demonstrate, this greedy algorithm works extremely
well in practice.\vspace{-0.15in}

\begin{algorithm}[H]
\centering{}%
\begin{minipage}[t]{0.9\columnwidth}%
\textbf{Initialization: $F_{0}=\emptyset$}

\textbf{For $t=1$ to $k$,}
\begin{alignat*}{1}
k_{t}^{*} & =\underset{k\in V\backslash F_{t-1}}{\arg\min}d(F_{t-1}\cup\{k\}),\ F_{t}=F_{t-1}\cup\{k_{t}^{*}\}
\end{alignat*}

\begin{center}
\caption{Selecting an FVS by a greedy approach}
\vspace{-0.1in}
\label{algo:greedynew}
\par\end{center}%
\end{minipage}
\end{algorithm}
\vspace{-0.15in}

\subsection{When the FVS Nodes Are Latent Variables}

When the feedback nodes are latent variables, the marginal distribution
of observed variables (the non-feedback nodes in the true model) has
information matrix $\tilde{J}_{T}=\hat{\Sigma}_{T}^{-1}=J_{T}-J_{M}J_{F}^{-1}J_{M}^{T}$.
If the exact $\tilde{J}_{T}$ is known, the learning problem is equivalent
to decomposing a given inverse covariance matrix $\tilde{J}_{T}$
into the sum of a tree-structured matrix $J_{T}$ and a rank-$k$
matrix $-J_{M}J_{F}^{-1}J_{M}^{T}$.%
\footnote{It is easy to see that different models having the same $J_{M}J_{F}^{-1}J_{M}$
cannot be distinguished using the samples, and thus without loss of
generality we can assume $J_{F}$ is normalized to be the identify
matrix in the final solution.%
} In general, use the ML criterion\vspace{-0.1in}

\begin{alignat}{1}
q_{\text{ML}}(\bx_{F},\bx_{T}) & =\arg\min_{q(\bx_{F},\bx_{T})\in Q_{F}}D_{\text{KL}}(\hat{p}(\bx_{T})||q(\bx_{T})),\label{eq:latent_objective}
\end{alignat}
\vspace{-0in}
where the optimization is over all nodes (latent and observed) while
the K-L divergence in the objective function is defined on the marginal
distribution of the observed nodes only. 

We propose \textit{the latent Chow-Liu algorithm, }an alternating
projection algorithm that is a variation of the EM algorithm and can
be viewed as an instance of the majorization-minimization algorithm.
The general form of the algorithm is as follows:
\begin{enumerate}
\item Project onto the empirical distribution:
\[
\hat{p}^{(t)}(\bx_{F},\bx_{T})=\hat{p}(\bx_{T})q^{(t)}(\bx_{F}|\bx_{T}).
\]

\item Project onto the best fitting structure on all variables:
\[
q^{(t+1)}(\bx_{F},\bx_{T})=\arg\min_{q(\bx_{F},\bx_{T})\in\calQ_{F}}D_{\text{KL}}(\hat{p}^{(t)}(\bx_{F},\bx_{T})||q(\bx_{F},\bx_{T})).
\]

\end{enumerate}
In the first projection, we obtain a distribution (on both observed
and latent variables) whose marginal (on the observed variables) matches
exactly the empirical distribution while maintaining the conditional
distribution (of the latent variables given the observed ones). In
the second projection we compute a distribution (on all variables)
in the family considered that is the closest to the distribution obtained
in the first projection. We found that among various EM type algorithms,
this formulation is the most revealing for our problems because it
clearly relates the second projection to the scenario where an FVS
$F$ is both observed and known (Section \ref{sub:knownFVS}). Therefore,
we are able to compute the second projection \textit{exactly} even
though the graph structure is \textit{unknown} (which allows \textit{any}
tree structure among the observed nodes). Note that when the feedback
nodes are latent, we do not need to select the FVS since it is simply
the set of latent nodes. This is the source of the simplification
when we use latent nodes for the FVS: We have no search of sets of
observed variables to include in the FVS. \vspace{-0.15in}

\begin{algorithm}[H]
\begin{centering}
\textbf{}%
\begin{minipage}[t]{0.95\columnwidth}%
\textbf{Input:} the empirical covariance matrix $\hat{\Sigma}_{T}$

\textbf{Output:} information matrix $J=\left[\begin{array}{cc}
J_{F} & J_{M}^{T}\\
J_{M} & J_{T}
\end{array}\right]$, where $J_{T}$ is tree-structured 
\begin{enumerate}
\item \noindent Initialization: $J^{(0)}=\left[\begin{array}{cc}
J_{F}^{(0)} & \left(J_{M}^{(0)}\right)^{T}\\
J_{M}^{(0)} & J_{T}^{(0)}
\end{array}\right]$.
\item \noindent Repeat for $t=1,2,3,\ldots$:

\begin{enumerate}
\item \noindent \textbf{P1}: Project to the empirical distribution:\\
$\hat{J}^{(t)}=\left[\begin{array}{cc}
J_{F}^{(t)} & (J_{M}^{(t)})^{T}\\
J_{M}^{(t)} & \left(\hat{\Sigma}_{T}\right)^{-1}+J_{M}^{(t)}(J_{F}^{(t)})^{-1}(J_{M}^{(t)})^{T}
\end{array}\right]$. Define $\hat{\Sigma}^{(t)}=\left(\hat{J}^{(t)}\right)^{-1}$.
\item \noindent \textbf{P2:} Project to the best fitting structure:\\
$\Sigma^{(t+1)}=\left[\begin{array}{cc}
\hat{\Sigma}_{F}^{(t)} & \left(\hat{\Sigma}_{M}^{(t)}\right)^{T}\\
\hat{\Sigma}_{M}^{(t)} & \text{CL}(\hat{\Sigma}_{T|F}^{(t)})+\hat{\Sigma}_{M}^{(t)}\left(\hat{\Sigma}_{F}^{(t)}\right)^{-1}\left(\hat{\Sigma}_{M}^{(t)}\right)^{T}
\end{array}\right]=\text{CCL}(\hat{\Sigma}^{(t)}),$\\
where $\hat{\Sigma}_{T|F}^{(t)}=\hat{\Sigma}_{T}^{(t)}-\hat{\Sigma}_{M}^{(t)}\left(\hat{\Sigma}_{F}^{(t)}\right)^{-1}\left(\hat{\Sigma}_{M}^{(t)}\right)^{T}$.
Define $J^{(t+1)}=\left(\Sigma^{(t+1)}\right)^{-1}.$
\end{enumerate}
\end{enumerate}
\caption{The latent Chow-Liu algorithm}

\label{algo:GaussEM}%
\end{minipage}
\par\end{centering}

\vspace{-0in}
\end{algorithm}
\vspace{-0.15in}

In Algorithm \ref{algo:GaussEM} we summarize the latent Chow-Liu
algorithm specialized for our family of GGMs, where both projections
have exact closed-form solutions and exhibit complementary structure---one
using the covariance and the other using the information parametrization.
In projection \textbf{P1}, three blocks of the information matrix
remain the same; In projection \textbf{P2,} three blocks of the covariance
matrix remain the same.

The two projections in Algorithm \ref{algo:GaussEM} can also be interpreted
as alternating \textit{low-rank} corrections : indeed,
\begin{eqnarray*}
\text{ }\text{In }\boldsymbol{\text{P1 }} & \hat{J}^{(t)} & =\left[\begin{array}{cc}
\bzero & \bzero\\
\bzero & \left(\hat{\Sigma}_{T}\right)^{-1}
\end{array}\right]+\left[\begin{array}{c}
J_{F}^{(t)}\\
J_{M}^{(t)}
\end{array}\right]\left(J_{F}^{(t)}\right)^{-1}\left[\begin{array}{cc}
J_{F}^{(t)} & \left(J_{M}^{(t)}\right)^{T}\end{array}\right],\\
\text{and in }\boldsymbol{\text{P2 }} & \Sigma^{(t+1)} & =\left[\begin{array}{cc}
\bzero & \bzero\\
\bzero & \text{CL}(\hat{\Sigma}_{T|F})
\end{array}\right]+\left[\begin{array}{c}
\hat{\Sigma}_{F}^{(t)}\\
\hat{\Sigma}_{M}^{(t)}
\end{array}\right]\left(\hat{\Sigma}_{F}^{(t)}\right)^{-1}\left[\begin{array}{cc}
\hat{\Sigma}_{F}^{(t)} & \left(\hat{\Sigma}_{M}^{(t)}\right)^{T}\end{array}\right],
\end{eqnarray*}
where the second terms of both expressions are of low-rank when the
size of the latent FVS is small. This formulation is the most intuitive
and simple, but a naive implementation of Algorithm \ref{algo:GaussEM}
has complexity ${\cal O}(n^{3})$ per iteration, where the bottleneck
is inverting full matrices $\hat{J}^{(t)}$ and $\Sigma^{(t+1)}$.
By carefully incorporating the inference algorithms into the projection
steps, we are able to further exploit the power of the models and
reduce the per-iteration complexity to $\calO(kn^{2}+n^{2}\log n)$,
which is the same as the complexity of the conditioned Chow-Liu algorithm
alone. We have the following proposition. 

\begin{proposition}

Using Algorithm \ref{algo:GaussEM}, the objective function of \eqref{eq:latent_objective}
decreases with the number of iterations, i.e., $D_{\text{\text{KL}}}(\calN(0,\hat{\Sigma}_{T})||\calN(0,\Sigma_{T}^{(t+1)}))\leq\calN(0,\hat{\Sigma}_{T})||\calN(0,\Sigma_{T}^{(t)}))$.
Using an accelerated version of Algorithm \ref{algo:GaussEM}, the
complexity per iteration is $\calO$($kn^{2}+n^{2}\log n$). 

\label{prop:LatenChowLiu}

\end{proposition}

Due to the page limit, we defer the description of the accelerated
version (\textit{the accelerated latent Chow-Liu algorithm}) and the
proof of Proposition \ref{prop:LatenChowLiu} to Appendix \ref{sec:Appen_Prop2}.
In fact, we never need to explicitly invert the empirical covariance
matrix $\hat{\Sigma}_{T}$ in the accelerated version. 

As a rule of thumb, we often use the spanning tree obtained by the
standard Chow-Liu algorithm as an initial tree among the observed
nodes. But note that \textbf{P2} involves solving a combinatorial
problem exactly, so the algorithm is able to jump among different
graph structures which reduces the chance of getting stuck at a bad
local minimum and gives us much more flexibility in initializing graph
structures. In the experiments, we will demonstrate that Algorithm
\ref{algo:GaussEM} is not sensitive to the initial graph structure.\vspace{-0.15in}

\section{Experiments\vspace{-0.15in}
}

In this section, we present experimental results for learning GGMs
with small FVSs, observed or latent, using both synthetic data and
real data of flight delays. \vspace{-0.15in}

\paragraph*{Fractional Brownian Motion: Latent FVS}

We consider a fractional Brownian motion (fBM) with Hurst parameter
$H=0.2$ defined on the time interval $(0,1]$. The covariance function
is $\Sigma(t_{1},t_{2})=\frac{1}{2}(|t_{1}|^{2H}+|t_{2}|^{2H}-|t_{1}-t_{2}|^{2H})$.
Figure \ref{fig:FBMfull} shows the covariance matrices of approximate
models using spanning trees (learned by the Chow-Liu algorithm), latent
trees (learned by the CLRG and NJ algorithms in \citep{choi2011learning})
and our latent FVS model (learned by Algorithm \ref{algo:GaussEM})
using 64 time samples (nodes). We can see that in the spanning tree
the correlation decays quickly (in fact exponentially) with distance,
which models the fBM poorly. The latent trees that are learned exhibit
blocky artifacts and have little or no improvement over the spanning
tree measured in the K-L divergence. In Figure \ref{fig:FBMplot-1},
we plot the K-L divergence (between the true model and the learned
models using Algorithm \ref{algo:GaussEM}) versus the size of the
latent FVSs for models with 32, 64, 128, and 256 time samples respectively.
For these models, we need about 1, 3, 5, and 7 feedback nodes respectively
to reduce the K-L divergence to 25\% of that achieved by the best
spanning tree model. Hence, we speculate that empirically $k=\calO(\log n)$
is a proper choice of the size of the latent FVS. We also study the
sensitivity of Algorithm \ref{algo:GaussEM} to the initial graph
structure. In our experiments, for different initial structures, Algorithm
\ref{algo:GaussEM} converges to the same graph structures (that give
the K-L divergence as shown in Figure \ref{fig:FBMplot-1}) within
three iterations.
\begin{figure}
\includegraphics[width=1\columnwidth]{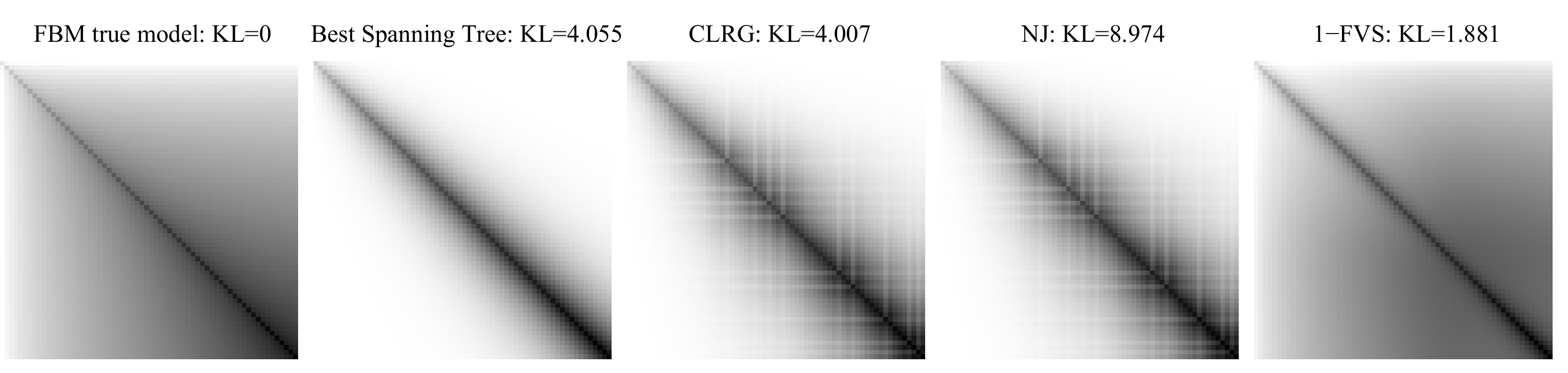}\caption{From left to right: 1) The true model (fBM with 64 time samples);
2) The best spanning tree; 3) The latent tree learned using the CLRG
algorithm in \citep{choi2011learning}; 4) The latent tree learned
using the NJ algorithm in \citep{choi2011learning}; 5) The model
with a size-one latent FVS learned using Algorithm \ref{algo:GaussEM}.
The gray scale is normalized for visual clarity. }

\label{fig:FBMfull}
\end{figure}
\begin{figure}
\vspace{-0.3in}
\begin{minipage}[t]{1\columnwidth}%
\subfloat[32 nodes]{\includegraphics[width=0.25\columnwidth]{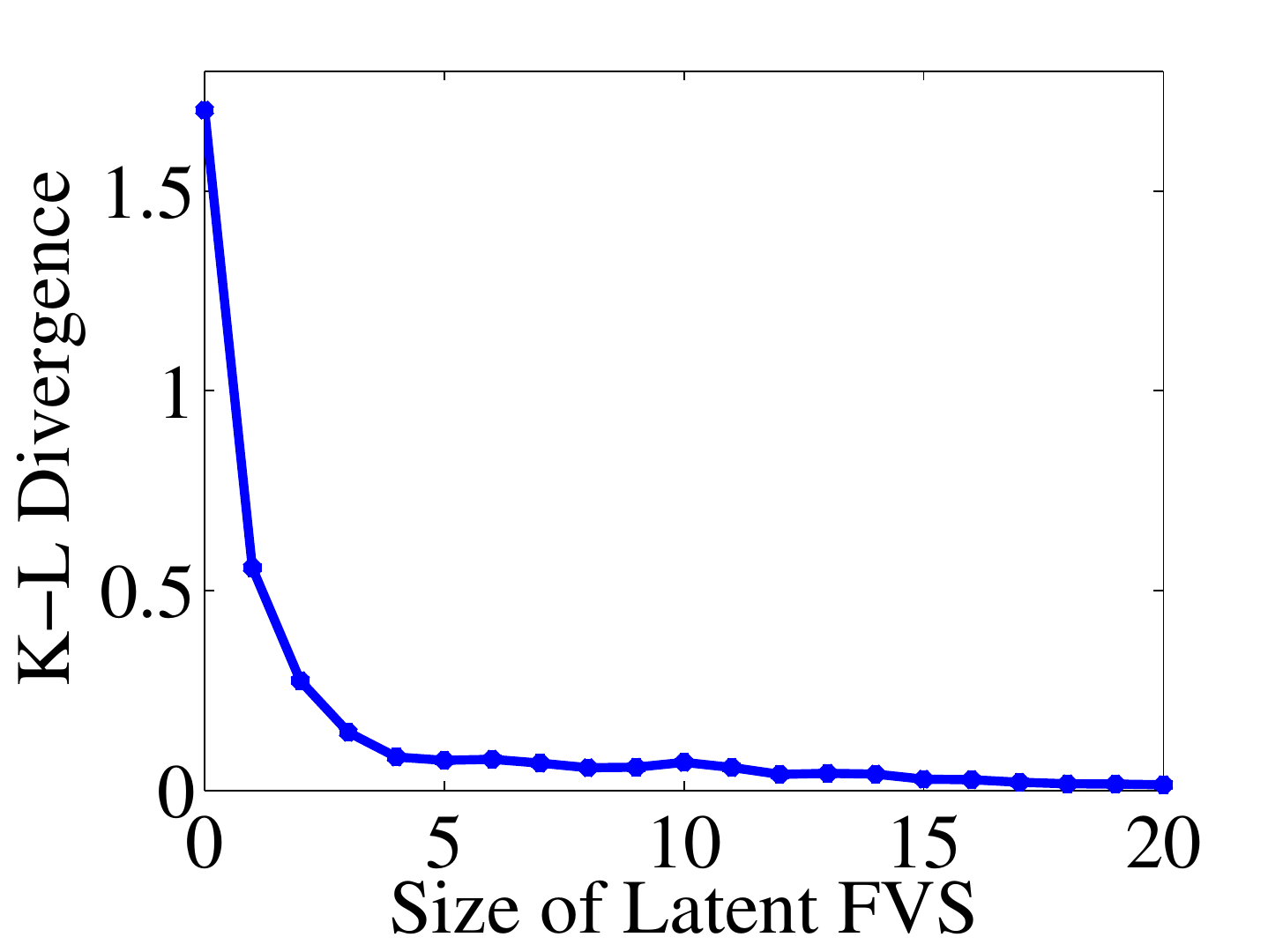}}\hfill{}\subfloat[64 nodes]{\includegraphics[width=0.25\columnwidth]{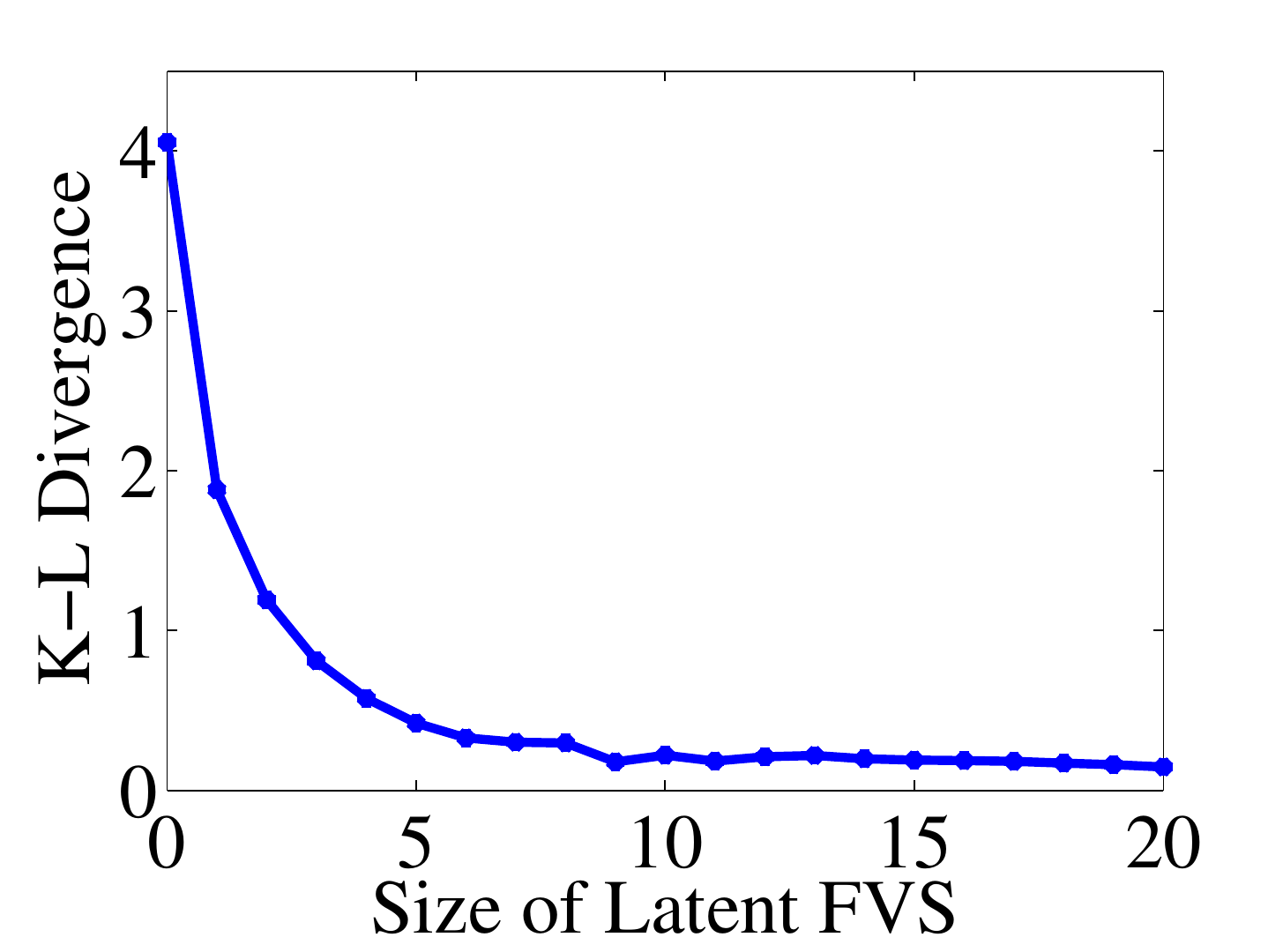}}\subfloat[128 nodes]{\includegraphics[width=0.25\columnwidth]{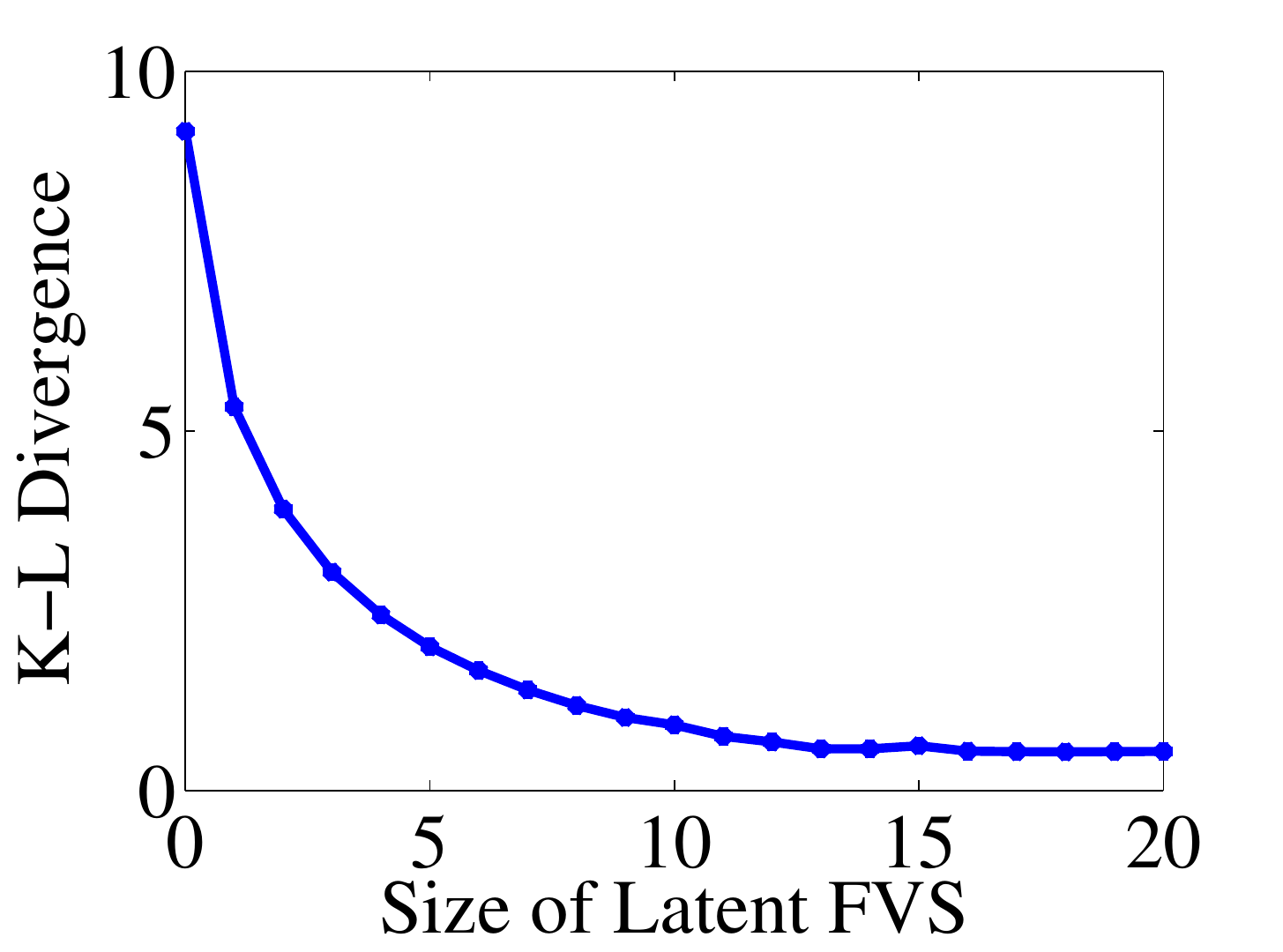}}\subfloat[256 nodes]{\includegraphics[width=0.25\columnwidth]{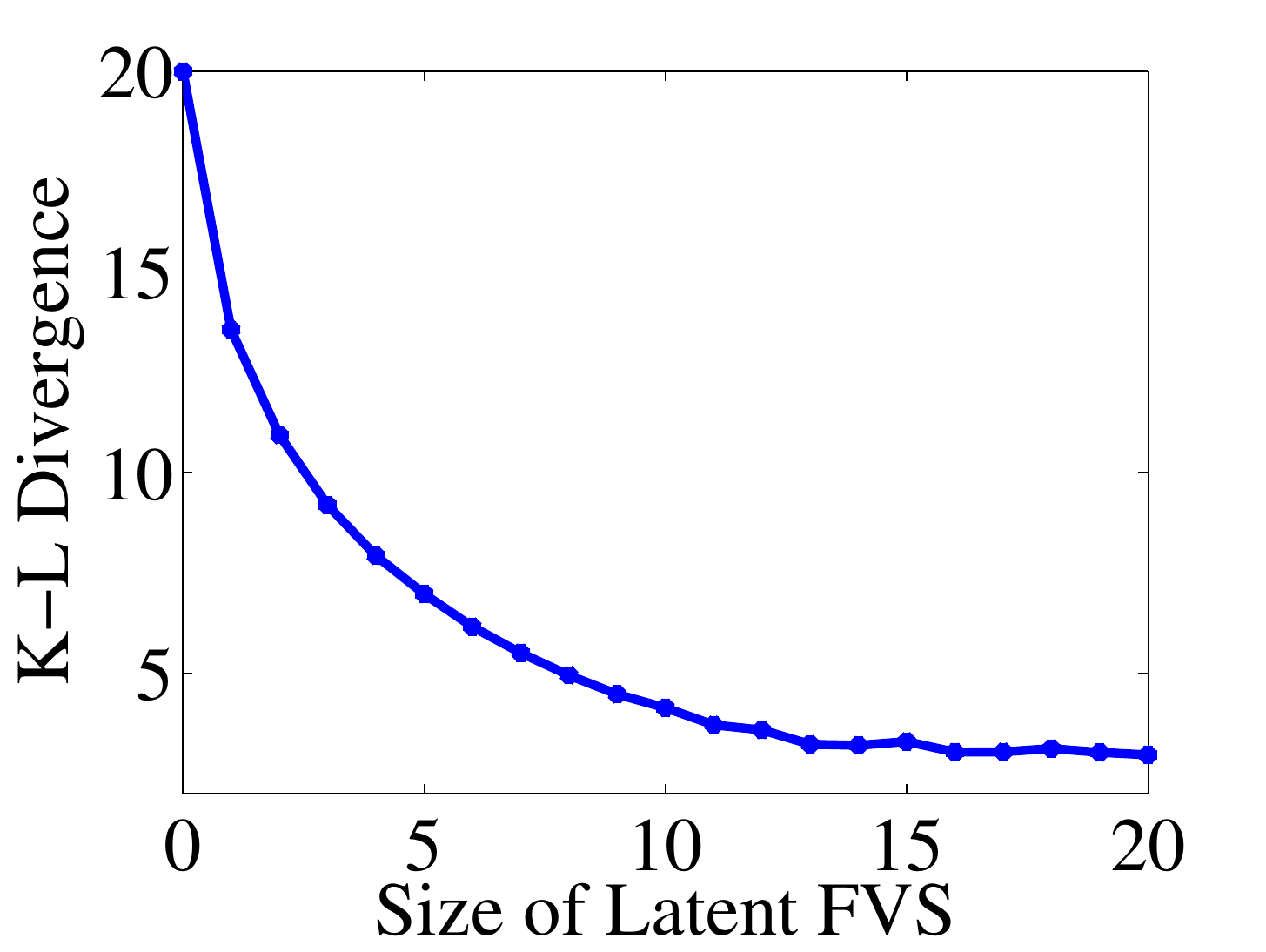}}\hfill{}

\caption{The relationship between the K-L divergence and the latent FVS size.
All models are learned using Algorithm \ref{algo:GaussEM} with 40
iterations.}

\label{fig:FBMplot-1}\vspace{-0.1in}
\end{minipage}

\vspace{-0.1in}
\end{figure}

\begin{figure}[t]
\subfloat[True Model]{\includegraphics[width=0.2\columnwidth]{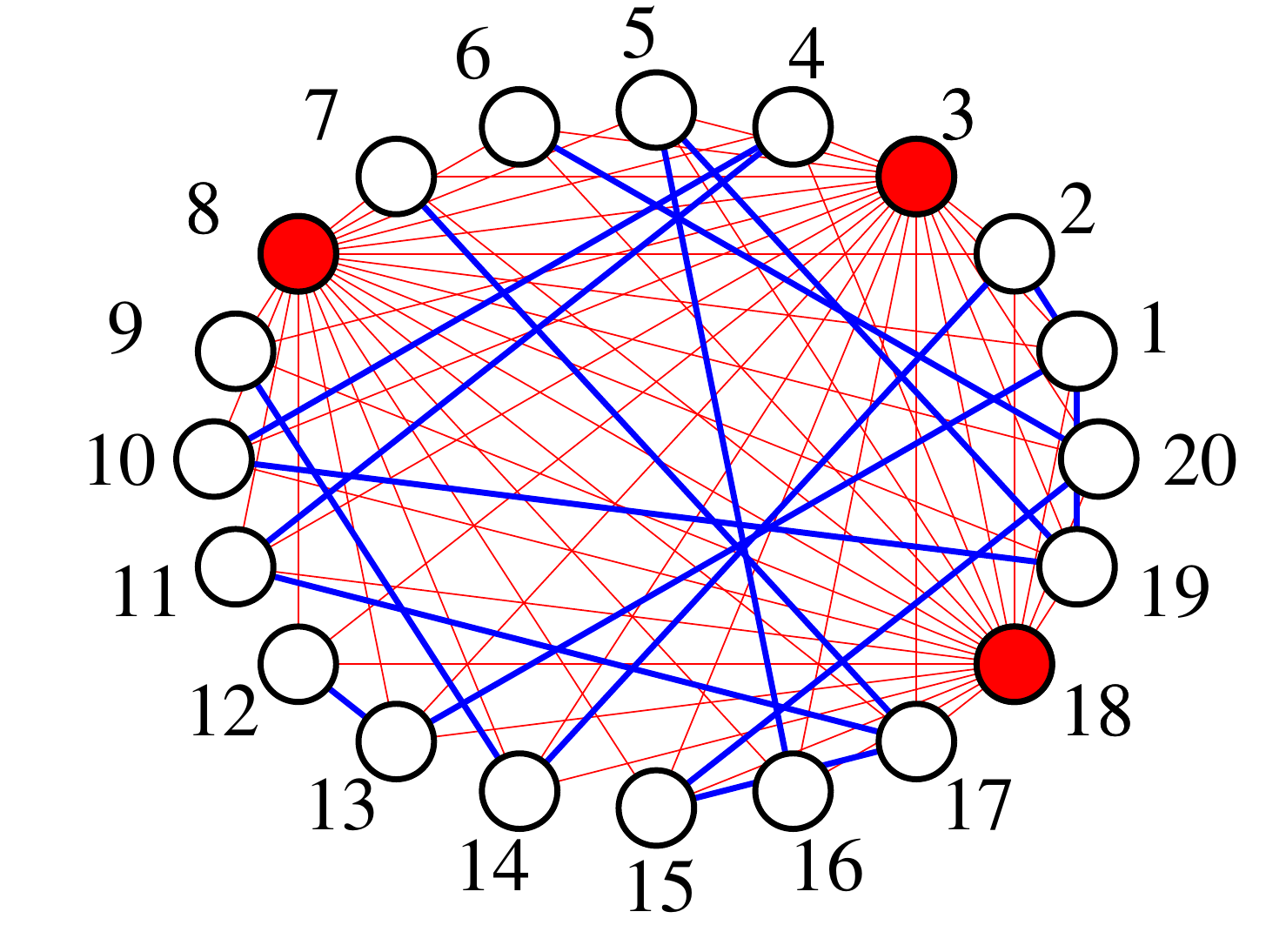}

}\subfloat[KL=12.7651]{\includegraphics[width=0.2\columnwidth]{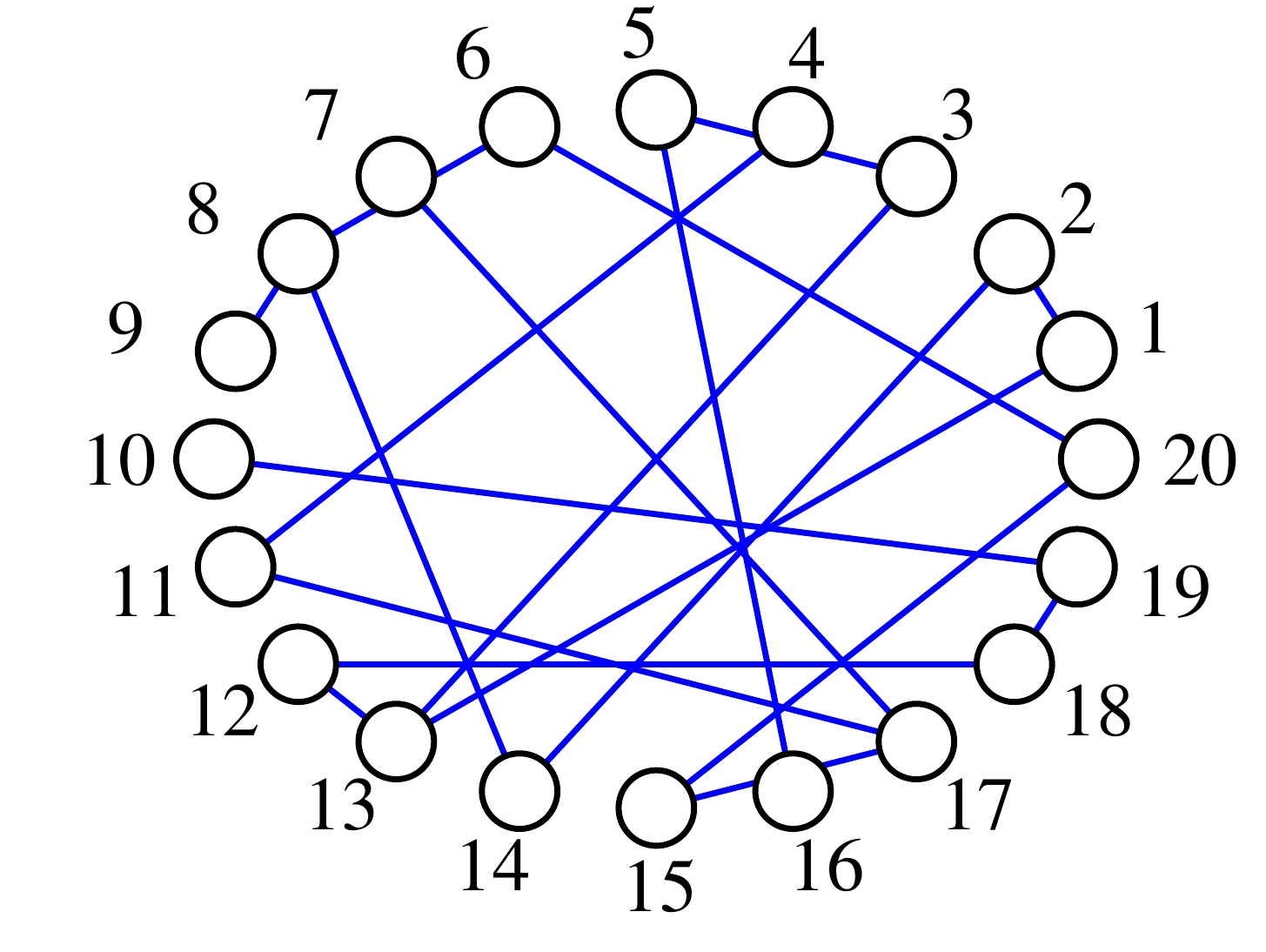}

}\subfloat[KL=1.3832]{\includegraphics[width=0.2\columnwidth]{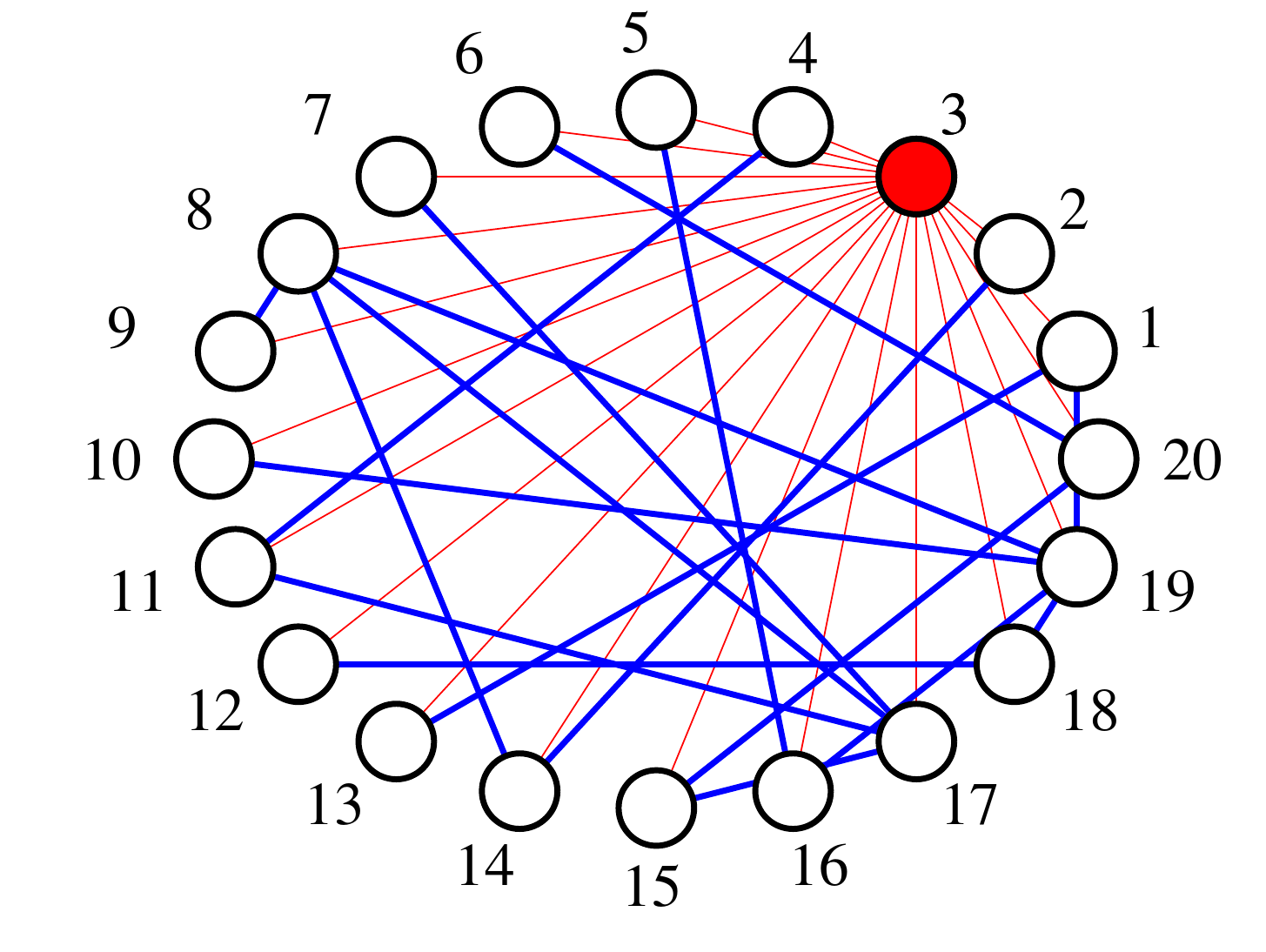}

}\subfloat[KL=0.6074]{\includegraphics[width=0.2\columnwidth]{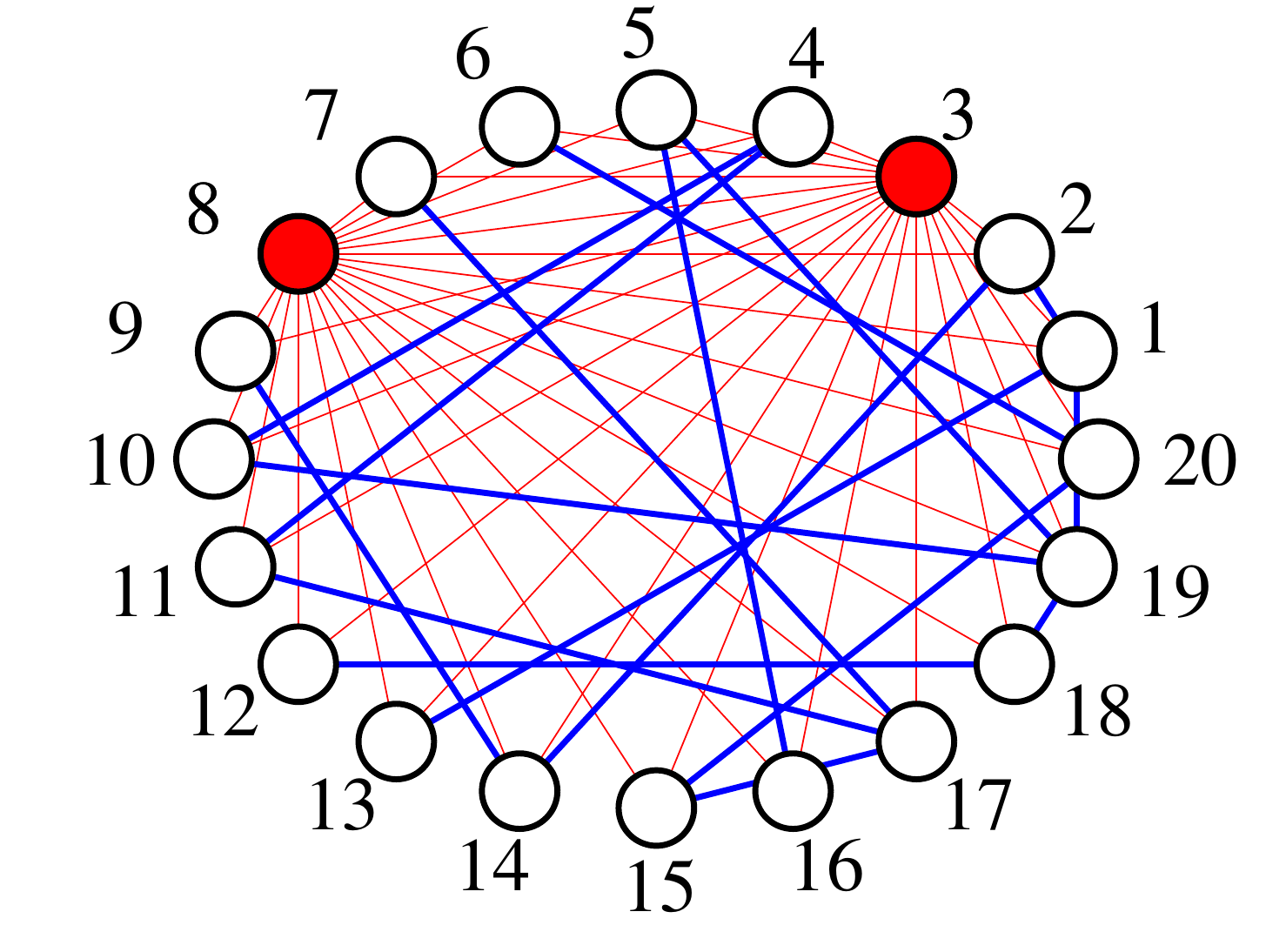}

}\subfloat[KL=0.0048]{\includegraphics[width=0.2\columnwidth]{Greedy3FVS}

\label{fig:3FVS}}

\caption{Learning a GGM using Algorithm \ref{algo:greedynew}. The thicker
blue lines represent the edges among the non-feedback nodes and the
thinner red lines represent other edges. (a) True model; (b) Tree-structured
model (0-FVS) learned from samples; (c) 1-FVS model; (d) 2-FVS model;
(e) 3-FVS model.}
\label{fig:greedy}
\end{figure}
\vspace{-0.15in}
\begin{figure}
\subfloat[Spanning Tree]{\includegraphics[width=0.23\columnwidth]{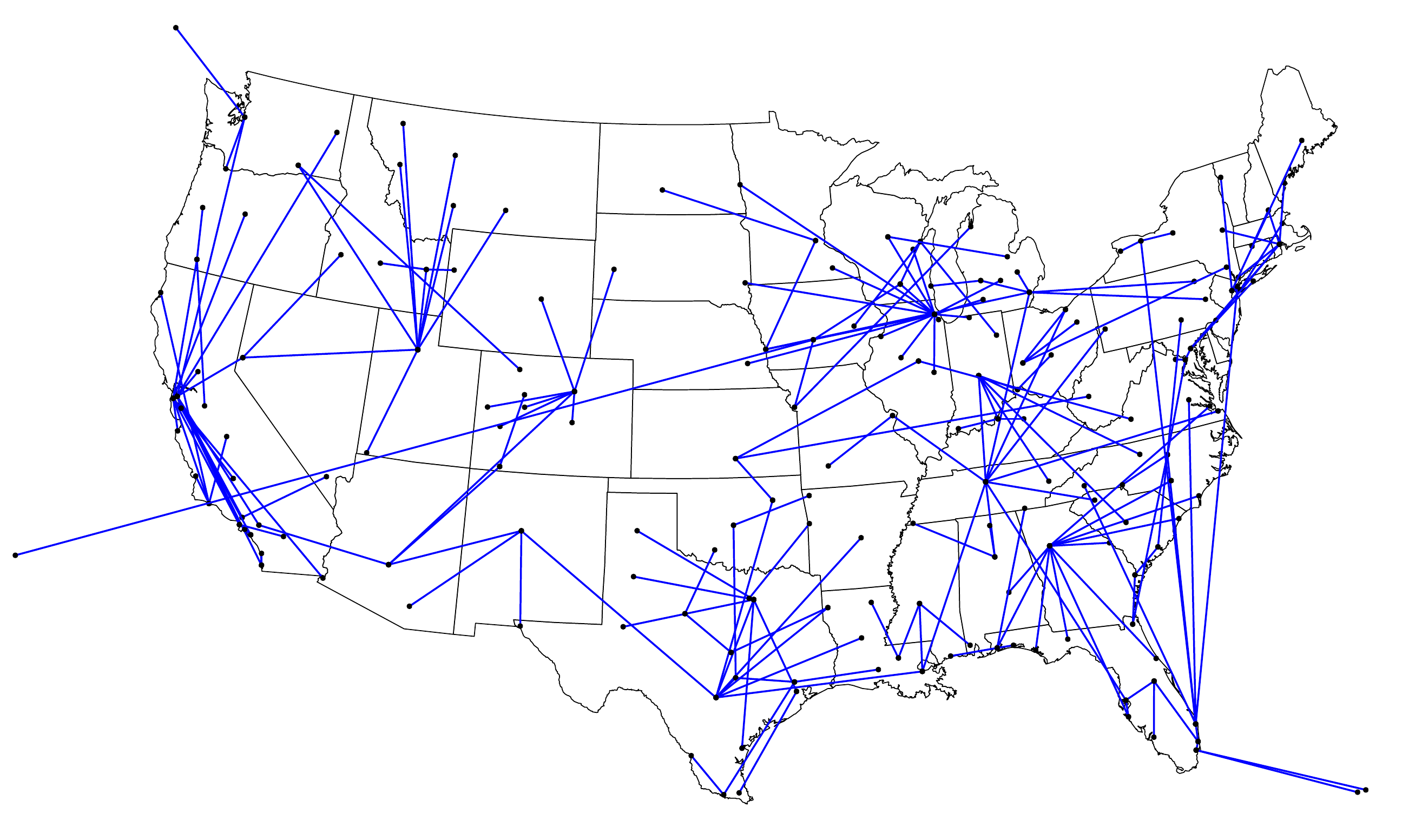}

\label{fig:spanningtree}}\hfill{}\subfloat[1-FVS GGM]{\includegraphics[width=0.23\columnwidth]{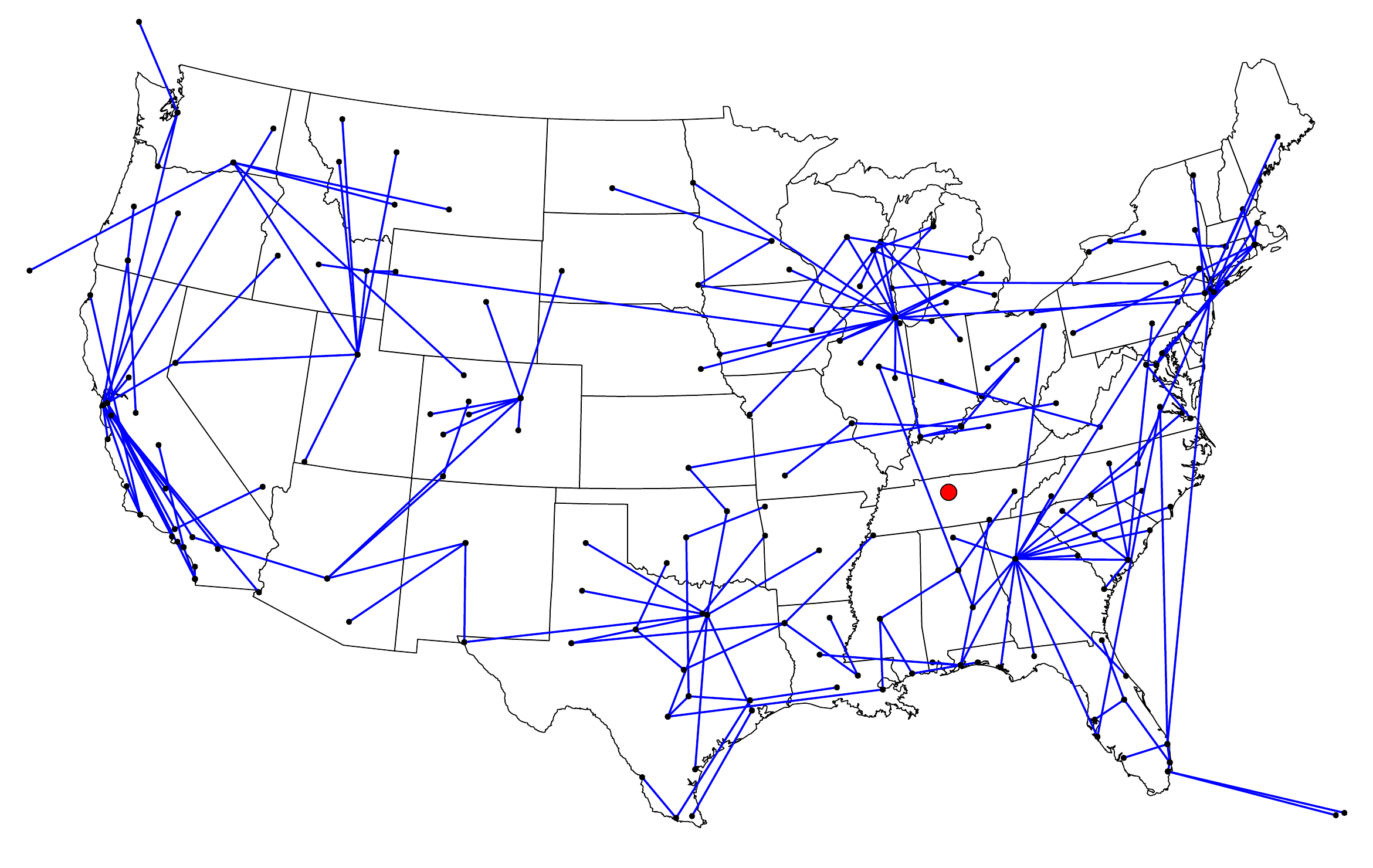}

\label{fig:1-fvstree}}\hfill{}\subfloat[3-FVS GGM]{\includegraphics[width=0.23\columnwidth]{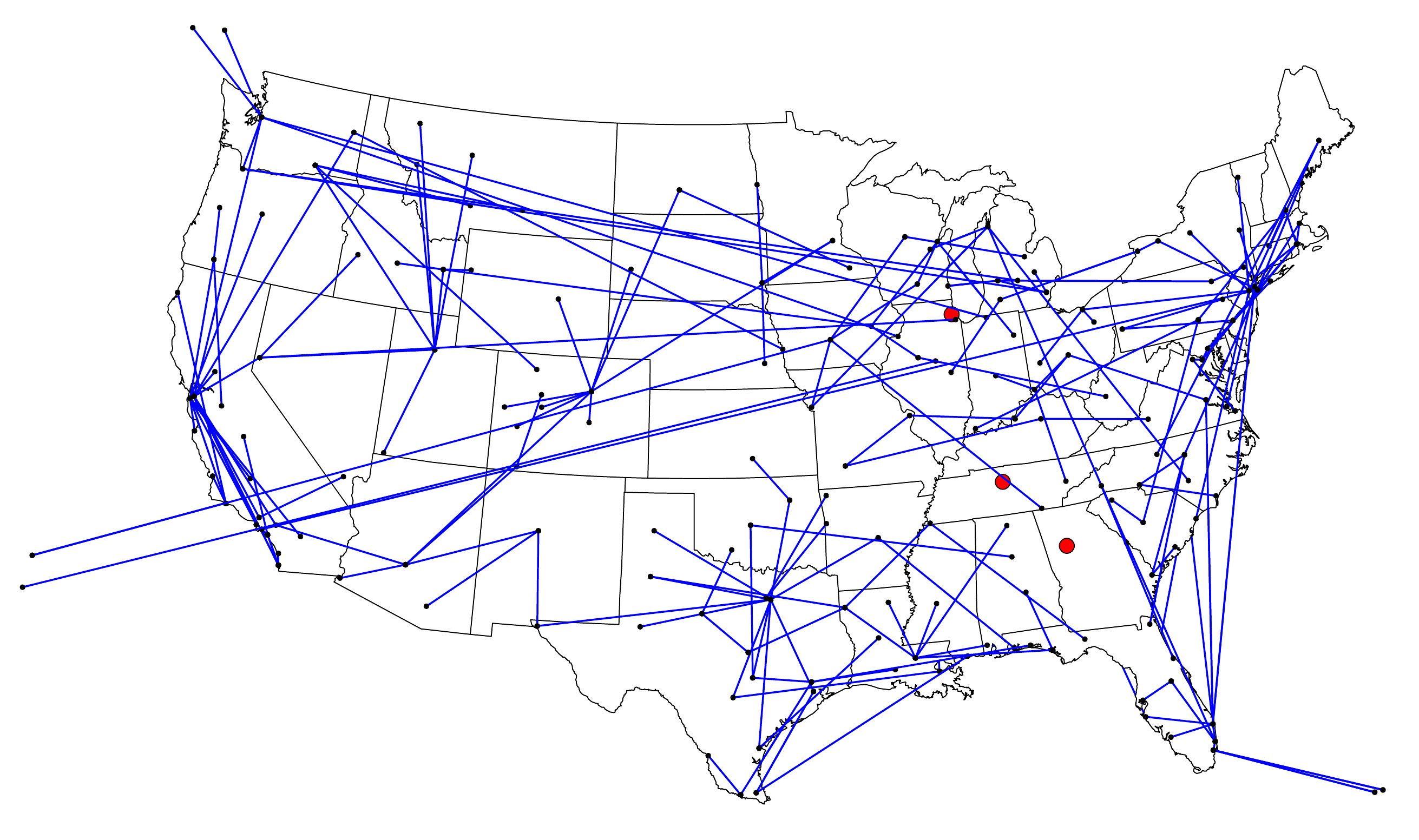}

\label{fig:3-fvstree}}\hfill{}\subfloat[10-FVS GGM]{\includegraphics[width=0.23\columnwidth]{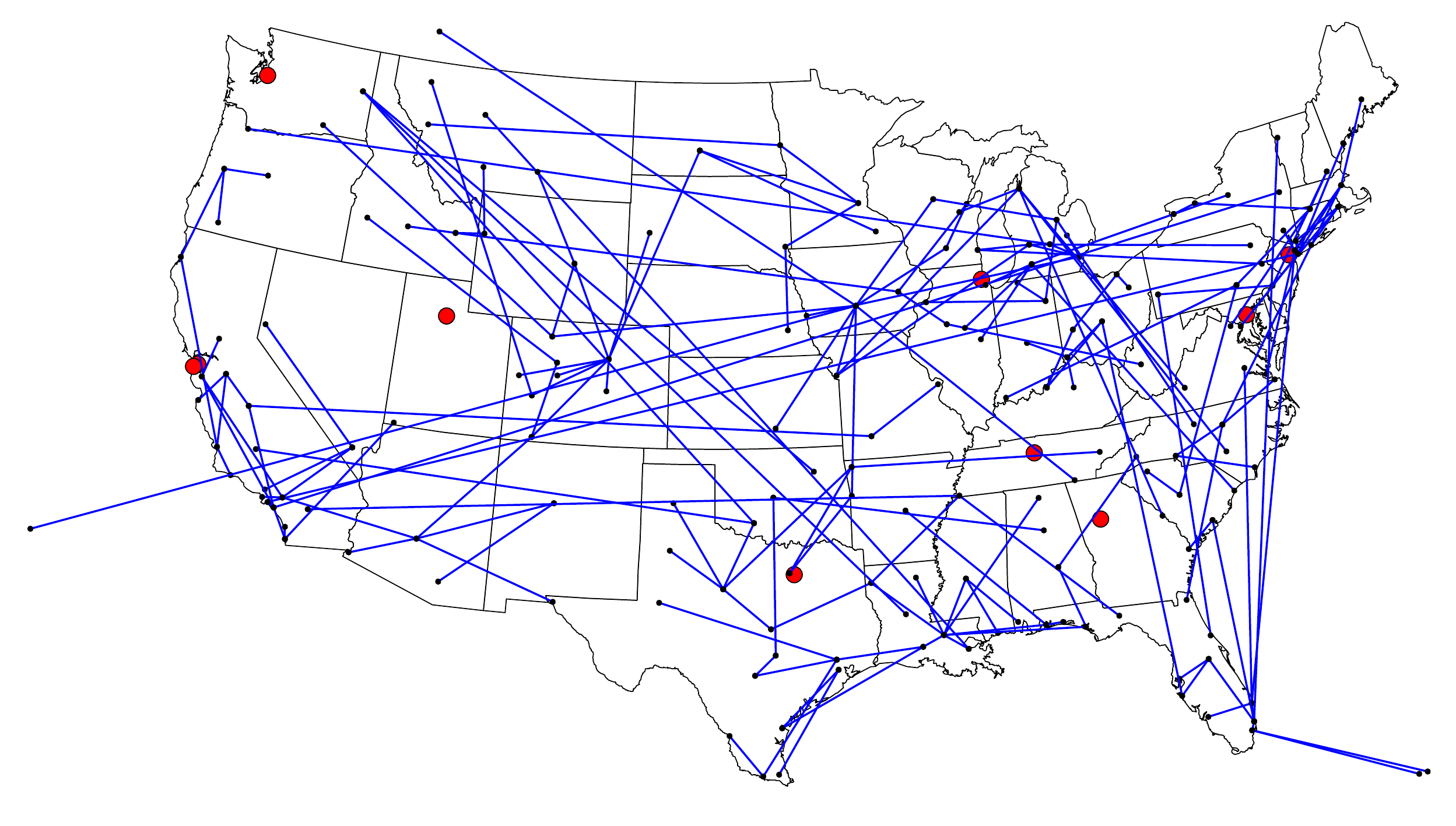}

\label{fig:10-fvstree}}

\caption{GGMs for modeling flight delays. The red dots denote selected feedback
nodes and the blue lines represent edges among the non-feedback nodes
(other edges involving the feedback nodes are omitted for clarity).}
\end{figure}
\vspace{-0.05in}

\paragraph*{Performance of the Greedy Algorithm: Observed FVS}

In this experiment, we examine the performance of the greedy algorithm
(Algorithm \ref{algo:greedynew}) when the FVS nodes are observed.
For each run, we construct a GGM that has 20 nodes and an FVS of size
three as the true model. We first generate a random spanning tree
among the non-feedback nodes. Then the corresponding information matrix
$J$ is also randomly generated: non-zero entries of $J$ are drawn
\textit{i.i.d. }from the uniform distribution $U[-1,1]$ with a multiple
of the identity matrix added to ensure $J\succ0$. From each generated
GGM, we draw 1000 samples and use Algorithm \ref{algo:greedynew}
to learn the model. For 100 runs that we have performed, we recover
the true graph structures successfully. Figure \ref{fig:greedy} shows
the graphs (and the K-L divergence) obtained using the greedy algorithm
for a typical run. We can see that we have the most divergence reduction
(from 12.7651 to 1.3832) when the first feedback node is selected.
When the size of the FVS increases to three (Figure \ref{fig:3FVS}),
the graph structure is recovered correctly. 

\vspace{-0.2in}

\paragraph*{Flight Delay Model: Observed FVS}

In this experiment, we model the relationships among airports for
flight delays. The raw dataset comes from RITA of the Bureau of Transportation
Statistics. It contains flight information in the U.S. from 1987 to
2008 including information such as scheduled departure time, scheduled
arrival time, departure delay, arrival delay, cancellation, and reasons
for cancellation for all domestic flights in the U.S. We want to model
how the flight delays at different airports are related to each other
using GGMs. First, we compute the average departure delay for each
day and each airport (of the top 200 busiest airports) using data
from the year 2008. Note that the average departure delays does not
directly indicate whether an airport is one of the major airports
that has heavy traffic. It is interesting to see whether major airports
(especially those notorious for delays) correspond to feedback nodes
in the learned models. Figure \ref{fig:spanningtree} shows the best
tree-structured graph obtained by the Chow-Liu algorithms (with input
being the covariance matrix of the average delay). Figure \ref{fig:1-fvstree}--\ref{fig:10-fvstree}
show the GGMs learned using Algorithm \ref{algo:greedynew}. It is
interesting that the first node selected is Nashville (BNA), which
is not one of the top \textquotedblleft{}hubs\textquotedblright{}
of the air system. The reason is that much of the statistical relationships
related to those hubs are approximated well enough, when we consider
a 1-FVS approximation, by a spanning tree (excluding BNA) and it is
the breaking of the cycles involving BNA that provide the most reduction
in K-L divergence over a spanning tree. Starting with the next node
selected in our greedy algorithm, we begin to see hubs being chosen.
In particular, the first ten airports selected in order are: BNA,
Chicago, Atlanta, Oakland, Newark, Dallas, San Francisco, Seattle,
Washington DC, Salt Lake City. Several major airports on the coasts
(e.g., Los Angeles and JFK) are not selected, as their influence on
delays at other domestic airports is well-captured with a tree structure.\vspace{-0.15in}

\section{Future Directions\vspace{-0.15in}
}

Our experimental results demonstrate the potential of these algorithms,
and, as in the work \citep{liu2012feedback}, suggests that choosing
FVSs of size $\calO(\log n)$ works well, leading to algorithms which
can be scaled to large problems. Providing theoretical guarantees
for this scaling (e.g., by specifying classes of models for which
such a size FVS provides asymptotically accurate models) is thus a
compelling open problem. In addition, incorporating complexity into
the FVS-order problem (e.g., as in AIC or BIC) is another direction
we are pursuing. Moreover, we are also working towards extending our
results to the non-Gaussian settings. \vspace{-0.15in}

\section*{Acknowledgments\vspace{-0.15in}
}

This research was supported in part by AFOSR under Grant FA9550-12-1-0287.\vspace{-0.15in}

\bibliographystyle{IEEEtran}
\bibliography{TotalReferences}

\begin{thebibliography}{10}
\providecommand{\url}[1]{#1}
\csname url@samestyle\endcsname
\providecommand{\newblock}{\relax}
\providecommand{\bibinfo}[2]{#2}
\providecommand{\BIBentrySTDinterwordspacing}{\spaceskip=0pt\relax}
\providecommand{\BIBentryALTinterwordstretchfactor}{4}
\providecommand{\BIBentryALTinterwordspacing}{\spaceskip=\fontdimen2\font plus
\BIBentryALTinterwordstretchfactor\fontdimen3\font minus
  \fontdimen4\font\relax}
\providecommand{\BIBforeignlanguage}[2]{{%
\expandafter\ifx\csname l@#1\endcsname\relax
\typeout{** WARNING: IEEEtran.bst: No hyphenation pattern has been}%
\typeout{** loaded for the language `#1'. Using the pattern for}%
\typeout{** the default language instead.}%
\else
\language=\csname l@#1\endcsname
\fi
#2}}
\providecommand{\BIBdecl}{\relax}
\BIBdecl

\bibitem{pearl1986constraint}
J.~Pearl, ``{A constraint propagation approach to probabilistic reasoning},''
  \emph{Proc. Uncertainty in Artificial Intell. (UAI)}, 1986.

\bibitem{chow1968approximating}
C.~Chow and C.~Liu, ``{Approximating discrete probability distributions with
  dependence trees},'' \emph{IEEE Trans. Inform. Theory}, vol.~14, no.~3, pp.
  462--467, 1968.

\bibitem{choi2009exploiting}
M.~Choi, V.~Chandrasekaran, and A.~Willsky, ``{Exploiting sparse Markov and
  covariance structure in multiresolution models},'' in \emph{Proc. 26th Annu.
  Int. Conf. on Machine Learning}.\hskip 1em plus 0.5em minus 0.4em\relax ACM,
  2009, pp. 177--184.

\bibitem{comer1999segmentation}
M.~Comer and E.~Delp, ``{Segmentation of textured images using a
  multiresolution Gaussian autoregressive model},'' \emph{IEEE Trans. Image
  Process.}, vol.~8, no.~3, pp. 408--420, 1999.

\bibitem{bouman1994multiscale}
C.~Bouman and M.~Shapiro, ``{A multiscale random field model for Bayesian image
  segmentation},'' \emph{IEEE Trans. Image Process.}, vol.~3, no.~2, pp.
  162--177, 1994.

\bibitem{karger2001learning}
D.~Karger and N.~Srebro, ``{Learning Markov networks: Maximum bounded
  tree-width graphs},'' in \emph{Proc. 12th Annu. ACM-SIAM Symp. on Discrete
  Algorithms}, 2001, pp. 392--401.

\bibitem{jordan2004graphical}
M.~Jordan, ``{Graphical models},'' \emph{Statistical Sci.}, pp. 140--155, 2004.

\bibitem{abbeel2006learning}
P.~Abbeel, D.~Koller, and A.~Ng, ``Learning factor graphs in polynomial time
  and sample complexity,'' \emph{J. Machine Learning Research}, vol.~7, pp.
  1743--1788, 2006.

\bibitem{dobra2004sparse}
A.~Dobra, C.~Hans, B.~Jones, J.~Nevins, G.~Yao, and M.~West, ``Sparse graphical
  models for exploring gene expression data,'' \emph{J. Multivariate Anal.},
  vol.~90, no.~1, pp. 196--212, 2004.

\bibitem{tipping2001sparse}
M.~Tipping, ``{Sparse Bayesian learning and the relevance vector machine},''
  \emph{J. Machine Learning Research}, vol.~1, pp. 211--244, 2001.

\bibitem{friedman2008sparse}
J.~Friedman, T.~Hastie, and R.~Tibshirani, ``Sparse inverse covariance
  estimation with the graphical lasso,'' \emph{Biostatistics}, vol.~9, no.~3,
  pp. 432--441, 2008.

\bibitem{ravikumar2008model}
P.~Ravikumar, G.~Raskutti, M.~Wainwright, and B.~Yu, ``{Model selection in
  Gaussian graphical models: High-dimensional consistency of l1-regularized
  MLE},'' \emph{Advances in Neural Information Processing Systems (NIPS)},
  vol.~21, 2008.

\bibitem{vazirani2004approximation}
V.~Vazirani, \emph{{Approximation Algorithms}}.\hskip 1em plus 0.5em minus
  0.4em\relax New York: Springer, 2004.

\bibitem{liu2012feedback}
Y.~Liu, V.~Chandrasekaran, A.~Anandkumar, and A.~Willsky, ``{Feedback message
  passing for inference in Gaussian graphical models},'' \emph{IEEE Trans.
  Signal Process.}, vol.~60, no.~8, pp. 4135--4150, 2012.

\bibitem{friedman1997bayesian}
N.~Friedman, D.~Geiger, and M.~Goldszmidt, ``Bayesian network classifiers,''
  \emph{Machine learning}, vol.~29, no.~2, pp. 131--163, 1997.

\bibitem{chandrasekaran2010latent}
V.~Chandrasekaran, P.~A. Parrilo, and A.~S. Willsky, ``Latent variable
  graphical model selection via convex optimization,'' in \emph{Communication,
  Control, and Computing (Allerton), 2010 48th Annual Allerton Conference
  on}.\hskip 1em plus 0.5em minus 0.4em\relax IEEE, 2010, pp. 1610--1613.

\bibitem{dinneen2001forbidden}
M.~Dinneen, K.~Cattell, and M.~Fellows, ``Forbidden minors to graphs with small
  feedback sets,'' \emph{Discrete Mathematics}, vol. 230, no.~1, pp. 215--252,
  2001.

\bibitem{brandt2011minimal}
F.~Brandt, ``Minimal stable sets in tournaments,'' \emph{J. Econ. Theory}, vol.
  146, no.~4, pp. 1481--1499, 2011.

\bibitem{bafna1992}
V.~Bafna, P.~Berman, and T.~Fujito, ``{A 2-approximation algorithm for the
  undirected feedback vertex set problem},'' \emph{SIAM J. Discrete
  Mathematics}, vol.~12, p. 289, 1999.

\bibitem{kirshner2004conditional}
S.~Kirshner, P.~Smyth, and A.~W. Robertson, ``{Conditional Chow-Liu tree
  structures for modeling discrete-valued vector time series},'' in
  \emph{Proceedings of the 20th conference on Uncertainty in artificial
  intelligence}.\hskip 1em plus 0.5em minus 0.4em\relax AUAI Press, 2004, pp.
  317--324.

\bibitem{choi2011learning}
M.~J. Choi, V.~Y. Tan, A.~Anandkumar, and A.~S. Willsky, ``Learning latent tree
  graphical models,'' \emph{Journal of Machine Learning Research}, vol.~12, pp.
  1729--1770, 2011.

\end{thebibliography}

\newpage{}

\appendix
\noindent \begin{center}
\textbf{\LARGE{Appendix of ``Learning Gaussian Graphical Models with
Observed or Latent FVSs''}}
\par\end{center}{\LARGE \par}

\section{Computing the Partition Function of GGMs in $\calQ_{F}$}

\label{sec:AppenPartitionFunction}

In Section \ref{sec:Gaussian-Graphical-Models} of the paper, we have
stated that given the information matrix $J$ of a GGM with an FVS
of size $k$, we can compute $\det J$ and hence the partition function
using a message-passing algorithm with complexity $\calO(k^{2}n)$.
This algorithm is inspired by the FMP algorithm developed in \citep{liu2012feedback}
and is described in Algorithm \ref{Algo:partitionfunction-1}. 

\begin{algorithm}[H]
\textbf{Input:} an FVS $F$ of size $k$ and an $n\times n$ information
matrix $J=\left[\begin{array}{cc}
J_{F} & J_{M}^{T}\\
J_{M} & J_{T}
\end{array}\right]$, where $J_{T}$ has tree structure $\calT$ with edge set $\calE_{T}$. 

\textbf{Output:} $\det J$
\begin{enumerate}
\item Run standard Gaussian BP on $\calT$ with information matrix $J_{T}$
to obtain $P{}_{ii}^{\calT}=\left(J_{T}^{-1}\right)_{ii}$ for all
$i\in T$, $P_{ij}^{\calT}=(J_{T}^{-1})_{ij}$ for all $(i,j)\in\calE_{T}$,
and $(\bg^{p})_{i}=(J_{T}^{-1}\bh^{p})_{i}$ for all $i\in T$ and
$p\in F$, where $\bh^{p}$ is the column of $J_{M}$ corresponding
to node $p$. 
\item Compute $\hat{J}_{F}$ with 
\[
\left(\hat{J}_{F}\right)_{pq}=J_{pq}-\sum_{j\in\calN(p)\cap T}J_{pj}g_{j}^{q},\ \forall\ p,q\in F
\]

\item Compute $\det\hat{J}_{F}$, the determinant of $\hat{J}_{F}$.
\item Output 
\[
\det J=\left(\prod_{(i,j)\in\calE_{T}}\frac{P_{ii}^{\calT}P_{jj}^{\calT}-\left(P_{ij}^{\calT}\right)^{2}}{P_{ii}^{\calT}P_{jj}^{\calT}}\prod_{i\in\calV}P_{ii}^{\calT}\right)^{-1}\det\hat{J}_{\calF}.
\]

\end{enumerate}
\caption{Computing the partition function when an FVS is given}

\label{Algo:partitionfunction-1}
\end{algorithm}

We state the correctness and the computational complexity of Algorithm
\ref{Algo:partitionfunction-1} in Proposition \ref{prop:partition}. 

\begin{proposition}

Algorithm \ref{Algo:partitionfunction-1} computes $\det J$ exactly
and the computational complexity is $\calO(k^{2}n)$.

\label{prop:partition}

\end{proposition}

Before giving the proof for Proposition \ref{prop:partition}, we
first prove Lemma \ref{lemma:treedet}. 

\begin{lemma}

If the information matrix $J\succ0$ has tree structure $\calT=(\calV,\calE)$,
then we have

\begin{equation}
\det\left(J\right)^{-1}=\prod_{i\in\calV}P_{ii}\prod_{(i,j)\in\calE}\frac{P_{ii}P_{jj}-P_{ij}^{2}}{P_{ii}P_{jj}},\label{eq:det_product-1}
\end{equation}
where $P=J^{-1}.$

\label{lemma:treedet}

\end{lemma}

\begin{proof}

WLOG, we assume the means are zero. For any tree-structured distribution
$p(\bx)$ with underlying tree $\calT$, we have the following factorization:

\begin{equation}
p(\bx)=\prod_{i\in\calV}p(x_{i})\prod_{(i,j)\in\calE_{\calT}}\frac{p(x_{i},x_{j})}{p(x_{i})p(x_{j})}.\label{eq:treefactor}
\end{equation}

For a GGM of $n$ nodes, the joint distribution, the singleton marginal
distributions, and the pairwise marginal distributions can be expressed
as follows. 
\begin{alignat*}{1}
p(\bx) & =\frac{1}{\left(2\pi\right)^{\frac{n}{2}}\left(\det J\right)^{-\frac{1}{2}}}\exp\{-\frac{1}{2}\bx^{T}J\bx\}\\
p(x_{i}) & =\frac{1}{(2\pi)^{\frac{1}{2}}P_{ii}{}^{\frac{1}{2}}}\exp\{-\frac{1}{2}\bx^{T}P_{ii}^{-1}\bx\}\\
p(x_{i},x_{j}) & =\frac{1}{2\pi\left(\det\left[\begin{array}{cc}
P_{ii} & P_{ij}\\
P_{ji} & P_{jj}
\end{array}\right]\right)^{\frac{1}{2}}}\exp\{-\frac{1}{2}\bx^{T}\left[\begin{array}{cc}
P_{ii} & P_{ij}\\
P_{ji} & P_{jj}
\end{array}\right]^{-1}\bx\}.
\end{alignat*}

Matching the normalization factors using \eqref{eq:treefactor}, we
obtain 
\begin{eqnarray}
\det\left(J\right)^{-1} & = & \prod_{i\in\calV}P_{ii}\prod_{(i,j)\in\calE}\frac{\det\left[\begin{array}{c}
\begin{array}{cc}
P_{ii} & P_{ij}\\
P_{ji} & P_{jj}
\end{array}\end{array}\right]}{P_{ii}P_{jj}}.\label{eq:partion_product}\\
 & = & \prod_{i\in\calV}P_{ii}\prod_{(i,j)\in\calE}\frac{P_{ii}P_{jj}-P_{ij}^{2}}{P_{ii}P_{jj}}
\end{eqnarray}

\end{proof}

Now we proceed to prove Proposition \ref{prop:partition}.

\begin{proof}

First, we show that $\hat{J}_{F}$ computed in Step 2 of Algorithm
\ref{Algo:partitionfunction-1} equals $J_{F}-J_{M}^{T}J_{T}^{-1}J_{M}$.
We have 
\[
\left[\begin{array}{cccc}
\bg^{1} & \bg^{2} & \cdots & \bg^{k}\end{array}\right]=J_{T}^{-1}\left[\begin{array}{cccc}
\bh^{1} & \bh^{2} & \cdots & \bh^{k}\end{array}\right]=J_{T}^{-1}J_{M}
\]
from the definition in Step 1. From Step 3, we can get 
\begin{alignat}{1}
\hat{J}_{F} & =J_{F}-\left[\begin{array}{cccc}
\bg^{1} & \bg^{2} & \cdots & \bg^{k}\end{array}\right]^{T}J_{T}\left[\begin{array}{cccc}
\bg^{1} & \bg^{2} & \cdots & \bg^{k}\end{array}\right]\nonumber \\
 & =J_{F}-\left(J_{T}^{-1}J_{M}\right)^{T}J_{T}\left(J_{T}^{-1}J_{M}\right)\nonumber \\
 & =J_{F}-J_{M}^{T}J_{T}^{-1}J_{M}.\label{eq:complement}
\end{alignat}
Hence, 
\begin{alignat}{1}
\det J & =\det\left(\left[\begin{array}{cc}
I & -J_{M}^{T}J_{T}^{-1}\\
\bzero & I
\end{array}\right]\right)\det\left(\left[\begin{array}{cc}
J_{F} & J_{M}^{T}\\
J_{M} & J_{T}
\end{array}\right]\right)\det\left(\left[\begin{array}{cc}
I & \bzero\\
-J_{T}^{-1}J_{M} & I
\end{array}\right]\right)\nonumber \\
 & =\det\left(\left[\begin{array}{cc}
I & -J_{M}^{T}J_{T}^{-1}\\
\bzero & I
\end{array}\right]\left[\begin{array}{cc}
J_{F} & J_{M}^{T}\\
J_{M} & J_{T}
\end{array}\right]\left[\begin{array}{cc}
I & \bzero\\
-J_{T}^{-1}J_{M} & I
\end{array}\right]\right)\nonumber \\
 & =\det\left[\begin{array}{cc}
J_{F}-J_{M}^{T}J_{T}^{-1}J_{M} & \bzero\\
\bzero & J_{T}
\end{array}\right]\nonumber \\
 & =\left(\det\hat{J}_{F}\right)\times\left(\det J_{T}\right),\label{eq:product}
\end{alignat}

From Lemma \ref{lemma:treedet} to follow, we have 
\begin{equation}
\det\left(J_{T}\right)^{-1}=\prod_{i\in\calV}P_{ii}^{\calT}\prod_{(i,j)\in\calE_{T}}\frac{P_{ii}^{\calT}P_{jj}^{\calT}-\left(P_{ij}^{\calT}\right)^{2}}{P_{ii}^{\calT}P_{jj}^{\calT}}.\label{eq:det_product-1-1}
\end{equation}

Hence, we have proved the correctness of the algorithm. Now we calculate
the complexity. The first step of Algorithm \ref{Algo:partitionfunction-1}
has complexity $\calO(n-k)$ using BP. Step 2 takes $\calO\left(k^{2}(n-k)\right)$
and the complexity of Step 3 is ${\cal O}(k^{3})$. Finally the complexity
of Step 4 is $\calO(n)$ since $\calT$ is a tree. The total complexity
is thus $\calO(k^{2}n)$. This completes the proof for Proposition\ref{prop:partition}.

\end{proof}

Note that if the FVS is not given, we can use the factor-2 approximate
algorithm in \citep{bafna1992} to obtain an FVS of size at most twice
the minimum size with complexity $\calO(\min\{m\log n,\ n^{2}\})$,
where $m$ is the number of edges.

\section{Proof for Proposition \ref{prop:givenF}}

\label{sec:Appen_Prop1}

\subsection{Preliminaries}

Proposition \ref{prop:givenF} states that Algorithm \ref{Algo:givenF}
computes the ML estimate with covariance $\Sigma_{\text{ML}}$ (together
with $\calE_{\text{ML}}$, the set of edges among the non-feedback
nodes) exactly with complexity $\calO(kn^{2}+n^{2}\log n)$, and that
$J_{\text{ML}}\stackrel{\Delta}{=}\Sigma_{\text{ML}}^{-1}$ can be
computed with additional complexity $\calO(k^{2}n)$.

First, we define the following information quantities:
\begin{enumerate}
\item The entropy $H_{p_{\bx}}(\bx)=-\int_{\bx}p_{\bx}(\bx)\log p_{\bx}(\bx)\text{d}\bx$
\item The conditional entropy $H_{p_{\bx,\by}}(\bx|\by)=-\int_{\bx,\by}p_{\bx,\by}(\bx,\by)\log p_{\bx|\by}(\bx|\by)\mathrm{d}\bx\mathrm{d}\by$
\item The mutual information $I_{p_{\bx,\by}}(\bx;\by)=\int_{\bx,\by}p_{\bx,\by}(\bx,\by)\log\frac{p(\bx)p(\by)}{p(\bx,\by)}\mathrm{d}\bx\mathrm{d}\by$
\item The conditional mutual information 
\[
I_{p_{\bx,\by,\bz}}(\bx;\by|\bz)=\int_{\bx,\by,\bz}p_{\bx,\by,\bz}(\bx,\by,\bz)\log\frac{p(\bx,\by|\bz)}{p(\bx|\bz)p(\by|\bz)}\mathrm{d}\bx\mathrm{d\by}
\]

\item The conditional K-L divergence: $D(\hat{p}_{\bx|\by}||q_{\bx|\by}|\hat{p}_{\by})\stackrel{\Delta}{=}D(\hat{p}_{\bx,\by}||q_{\bx|\by}\hat{p}_{y})$.
\end{enumerate}
The (conditional) K-L divergence is always nonnegative. It is zero
if and only if the two distributions are the same (almost everywhere).
When there is no confusion, the subscripts in the distributions are
often omitted, e.g., $I_{p_{\bx,\by}}(\bx;\by)$ written as $I_{p}(\bx;\by)$.
With a slight abuse of notation, l we use $p(\bx_{F})$ to denote
the marginal distribution of $\bx_{F}$ under the joint distribution
$p(\bx)$, and similarly $p(\bx_{T}|\bx_{F})$ to denote the conditional
distribution of $\bx_{T}$ given $\bx_{F}$ under the joint distribution
$p(\bx)$.

The standard Chow-Liu Algorithm for GGMs is summarized in \ref{algo:CL-1}.
The complexity is $\calO(n^{2}\log n)$. Note that in Step 3, for
a fixed $i$, for any $(i,j)\notin\calE_{T}$, $\Sigma_{ij}$ can
be computed following a topological order of with $i$ being the root.
Hence, by book-keeping the computed products along the paths, the
complexity of computing each $\Sigma_{ij}$ is $\calO(1)$. 

\begin{algorithm}[H]
\textbf{Input: }the empirical covariance matrix\textbf{ $\hat{\Sigma}$}

\textbf{Output:} $\Sigma_{\text{CL}}$ and $\calE_{\text{CL}}$
\begin{enumerate}
\item Compute the correlation coefficients $\rho_{ij}=\frac{\hat{\Sigma}_{ij}}{\sqrt{\hat{\Sigma}_{ii}\hat{\Sigma}_{jj}}}$
\item Find an MST (maximum weight spanning tree) of the complete graph with
weights $|\rho_{ij}|$ for edge $(i,j)$. The edge set of the tree
is denoted as $\calE_{T}$.
\item For all $i\in{\cal V},$ $\left(\Sigma_{\text{CL}}\right)_{ii}=\hat{\Sigma}_{ii}$;
for $(i,j)\in\calE_{T}$, $\left(\Sigma_{\text{CL}}\right)_{ij}=\hat{\Sigma}_{ij}$;
for $(i,j)\notin\calE_{T}$, $\left(\Sigma_{\text{CL}}\right)_{ij}=\sqrt{\Sigma_{ii}\Sigma_{jj}}\prod_{(l,k)\in\text{Path}(i,j)}\rho_{lk}$,
where $\text{Path}(i,j)$ is the set of edges on the unique path between
$i$ and $j$ in the spanning tree.
\end{enumerate}
\caption{the Chow-Liu Algorithm for GGMs}

\label{algo:CL-1}
\end{algorithm}

\subsection{Lemmas}

Lemma \ref{lemma:treefactorization} is a well-known result stated
without proof.

\begin{lemma}

The p.d.f. of a tree-structured model $\calT=(\calV,\calE)$ can be
factorized according to either of the following two equations: 
\begin{enumerate}
\item $p(\bx)=p(\mathrm{x}_{r})\prod_{i\in\calV\backslash r}p(\rx_{i}|\rx_{\pi(i)}),$
where $r$ is an arbitrary node selected as the root and $\pi(i)$
is the unique parent of node $i$ in the tree rooted at $r$. 
\item $p(\bx)=\prod_{i\in\calV}p(\rx_{i})\prod_{(i,j)\in\calE}\frac{p(\rx_{i},\rx_{j})}{p(\rx_{i})p(\rx_{j})}.$
\end{enumerate}
\label{lemma:treefactorization}

\end{lemma}

For a given $F$ and a fixed tree $\calT$ with edge set $\calE_{T}$
among the non-feedback nodes, Lemma \ref{lemma:KLexpression-1} gives
a closed form solution that minimizes the K-L divergence.

\begin{lemma}

\begin{equation}
\min_{q\in\calQ_{F,\calT}}D_{\text{KL}}(\hat{p}||q)=-H_{\hat{p}}(\bx)+H_{\hat{p}}(\bx_{F})+\sum_{i\in\calV\backslash F}H_{\hat{p}}(\bx_{i}|\bx_{F})-\sum_{(i,j)\in\calE_{\calT}}I_{\hat{p}}(\bx_{i};\bx_{j}|\bx_{F}),\label{eq:KLfixedtree}
\end{equation}
where $\calQ_{F,\calT}$ is the set of distributions defined on a
graph with a given FVS $F$ and a given spanning tree $\calT$ among
the non-feedback nodes. The minimum K-L divergence is obtained if
and only if: 1) $q(\bx_{F})=\hat{p}(\bx_{F})$; 2) $q({\bf x}_{F},x_{i},x_{j})=\hat{p}(\bx_{F},x_{i},x_{j})$
for any $(i,j)\in\calE_{\calT}$.

\label{lemma:KLexpression-1}

\end{lemma}

\begin{proof}

With fixed $F$ and $\calT$, 
\begin{align}
D_{\text{KL}}(\hat{p}||q) & =\int\hat{p}(\bx)\log\frac{\hat{p}(\bx)}{q(\bx)}\mathrm{d}\bx\nonumber \\
 & =-H_{\hat{p}}(\bx)-\int\hat{p}(\bx)\log q(\bx)\mathrm{d}\bx\nonumber \\
 & =-H_{\hat{p}}(\bx)-\int\hat{p}(\bx)\log\left(q(\bx_{F})q(\bx_{T}|\bx_{F})\right)\mathrm{d}\bx\nonumber \\
 & \stackrel{(a)}{=}-H_{\hat{p}}(\bx)-\int\hat{p}(\bx)\log\left(q(\bx_{F})q(\bx_{r}|\bx_{F})\prod_{i\in\calV\backslash F\backslash r}q(\bx_{i}|\bx_{F},\bx_{\pi(i)})\right)\mathrm{d}\bx\nonumber \\
 & =-H_{\hat{p}}(\bx)-\int\hat{p}(\bx_{F})\log q(\bx_{F})\mathrm{d}\bx_{F}-\int\hat{p}(\bx_{F},\bx_{r})\log q(\bx_{r}|\bx_{F})\mathrm{d}\bx_{F}\mathrm{d}\bx_{r}\nonumber \\
 & \quad-\sum_{i\in\calV\backslash F\backslash r}\int\hat{p}(\bx_{F},\bx_{\pi(i)},\bx_{i})\log q(\bx_{i}|\bx_{F},\bx_{\pi(i)})\mathrm{d}\bx_{F}\mathrm{d}\bx_{\pi(i)}\mathrm{d}\bx_{i}\nonumber \\
 & \stackrel{(b)}{=}-H_{\hat{p}}(\bx)+H_{\hat{p}}(\bx_{F})+D(\hat{p}_{F}||q_{F})+H_{\hat{p}}(\bx_{r}|\bx_{F})+D(\hat{p}_{r|F}||q_{r|F}|\hat{p}_{F})\nonumber \\
 & \quad+\sum_{i\in V\backslash F\backslash r}H_{\hat{p}}(\bx_{i}|\bx_{F,\pi(i)})+D(\hat{p}_{i|F,r}||q_{i|F,r}|\hat{p}_{F,r})\nonumber \\
 & \stackrel{(c)}{\geq}-H_{\hat{p}}(\bx)+H_{\hat{p}}(\bx_{F})+H_{\hat{p}}(\bx_{r}|\bx_{F})+\sum_{i\in V\backslash F\backslash r}H_{\hat{p}}(\bx_{i}|\bx_{F,\pi(i)}),\label{eq:expression1-1}
\end{align}
where (a) is obtained by using Factorization 1 in Lemma \ref{lemma:treefactorization}
with an arbitrary root node $r$; (b) can be directly verified using
the definition of the information quantities, and the equality in
(c) is satisfied when $q_{F}=\hat{p}_{F}$, $q_{r|F}=\hat{p}_{r|F}$,
and $q_{i|F,\pi(i)}=\hat{p}_{i|F,\pi(i)},\forall i\in T\backslash r$,
or equivalently when 
\begin{align}
q_{F} & =\hat{p}_{F}\nonumber \\
q_{F,i,j} & =\hat{p}_{F,i,j},\forall(i,j)\in\calE_{\calT}.\label{eq:optimalparameter-1}
\end{align}

Next, we derive another expression of \eqref{eq:expression1-1}. By
substituting \eqref{eq:optimalparameter-1} into Factorization s of
Lemma \ref{lemma:treefactorization}, we have

\begin{eqnarray*}
q^{*}(\bx) & = & \hat{p}(\bx_{F})\prod_{i\in T}\hat{p}(x_{i}|\bx_{F})\prod_{(i,j)\in\calE_{\calT}}\frac{\hat{p}(\bx_{i},\bx_{j}|\bx_{F})}{\hat{p}(\bx_{i}|\bx_{F})\hat{p}(\bx_{j}|\bx_{F})}.
\end{eqnarray*}
Hence, 
\begin{alignat}{1}
\min_{q\in\calQ_{F,\calT}}D(\hat{p}||q) & =D(\hat{p}||q^{*})\nonumber \\
 & -H_{\hat{p}}(\bx)+H_{\hat{p}}(\bx_{F})+\sum_{i\in V\backslash F}H_{\hat{p}}(\bx_{i}|\bx_{F})\\
 & \ \ \ \ +\sum_{(i,j)\in\calE_{\calT}}\int\hat{p}_{F,i,j}(\bx_{F},\bx_{i},\bx_{j})\log\frac{\hat{p}(\bx_{i},\bx_{j}|\bx_{F})}{\hat{p}(\bx_{i}|\bx_{F})\hat{p}(\bx_{j}|\bx_{F})}\mathrm{d}\bx_{F}\mathrm{d}\bx_{i}\mathrm{d}\bx_{j}\\
 & =H_{\hat{p}}(\bx)+H_{\hat{p}}(\bx_{F})+\sum_{i\in V\backslash F}H_{\hat{p}}(\bx_{i}|\bx_{F})\\
 & \ \ \ \ -\sum_{(i,j)\in\calE_{\calT}}\int\hat{p}_{F,i,j}(\bx_{F},\bx_{i},\bx_{j})\log\frac{\hat{p}(\bx_{i}|\bx_{F})\hat{p}(\bx_{j}|\bx_{F})}{\hat{p}(\bx_{i},\bx_{j}|\bx_{F})}\mathrm{d}\bx_{F}\mathrm{d}\bx_{i}\mathrm{d}\bx_{j}\\
 & =-H_{\hat{p}}(\bx)+H(\hat{p}_{F})+\sum_{i\in\calV\backslash F}H(\hat{p}_{i|F}|\bx_{F})-\sum_{(i,j)\in\calE_{\calT}}I_{\hat{p}}(\bx_{i};\bx_{j}|\bx_{F}).\label{eq:KLexpression-1}
\end{alignat}

We have thus proved Lemma \ref{lemma:KLexpression-1}.

\end{proof}

The following Lemma \ref{lemma:KLGaussian} gives a closed-form expression
for the K-L divergence between two Gaussians. It can be verified by
calculus and the proof is omitted. 

\begin{lemma}

\textcolor{black}{For two $n$-dimensional Gaussian distributions
$\hat{p}(\bx)=\calN(\bx;\hat{\bmu},\hat{\Sigma})$ and $q(\bx)=\calN(\bx;\bmu,\Sigma)$,
we have}

\begin{equation}
D(\hat{p}||q)=\frac{1}{2}\left(\text{Tr\ensuremath{\left(\Sigma^{-1}\hat{\Sigma}\right)}+\ensuremath{\left(\bmu-\hat{\bmu}\right)}}^{T}\Sigma^{-1}\left(\bmu-\hat{\bmu}\right)-n\ln\det\left(\Sigma^{-1}\hat{\Sigma}\right)\right).\label{eq:KLGaussian-1}
\end{equation}

\label{lemma:KLGaussian}

\end{lemma}

An immediate implication of Lemma \ref{lemma:KLGaussian} is that
when learning GGMs we always have that $\bmu_{\text{ML}}=\hat{\bmu}$
if there is no constraint on the mean in. 

\begin{lemma}

If a symmetric positive definite matrix $\Sigma$ is given and we
know that its inverse $J=\Sigma^{-1}$ is sparse with respect to a
tree $\calT=(\calV,\calE)$, then the non-zero entries of $J$ can
be computed using \eqref{eq:invertTree} in time $\calO(n)$.

\begin{equation}
J_{ij}=\begin{cases}
\left(1-\text{deg}(i)\right)\Sigma_{ii}^{-1}+\sum_{j\in\calN(i)}\left(\Sigma_{ii}-\Sigma_{ij}\Sigma_{jj}^{-1}\Sigma_{ji}\right)^{-1} & i=j\in\calV\\
\frac{\Sigma_{ij}}{\Sigma_{ij}^{2}-\Sigma_{ii}\Sigma_{jj}} & (i,\ j)\in\calE\\
0 & \text{otherwise},
\end{cases}\label{eq:invertTree}
\end{equation}

where $\calN(i)$ is the set of neighbors of node $i$ in $\calT$;
$\text{deg}(i)$ is the degree of $i$ in $\calT$.

\label{lemma:knn}

\end{lemma}

\begin{proof}

Since $\Sigma\succ0$, we can construct a Gaussian distribution $p(\bx)$
with zero mean and covariance matrix $\Sigma$. The distribution is
tree-structured because $J=\Sigma^{-1}$ has tree structure $\calT$.
Hence, we have the following factorization. 

\[
p(\bx)=\prod_{i\in\calV}p(x_{i})\prod_{(i,j)\in\calE}\frac{p(x_{i},x_{j})}{p(x_{i})p(x_{j})},
\]
where 
\begin{alignat*}{1}
p(\bx) & =\frac{1}{\left(2\pi\right)^{\frac{n}{2}}\left(\det J\right)^{-\frac{1}{2}}}\exp\{-\frac{1}{2}\bx^{T}J\bx\}\\
p(x_{i}) & =\frac{1}{(2\pi)^{\frac{1}{2}}P_{ii}{}^{\frac{1}{2}}}\exp\{-\frac{1}{2}\bx^{T}\Sigma_{ii}^{-1}\bx\}\\
p(x_{i},x_{j}) & =\frac{1}{2\pi\left(\det\left[\begin{array}{cc}
\Sigma_{ii} & \Sigma_{ij}\\
\Sigma_{ji} & \Sigma_{jj}
\end{array}\right]\right)^{\frac{1}{2}}}\exp\{-\frac{1}{2}\bx^{T}\left[\begin{array}{cc}
\Sigma_{ii} & \Sigma_{ij}\\
\Sigma_{ji} & \Sigma_{jj}
\end{array}\right]^{-1}\bx\}.
\end{alignat*}

By matching the quadratic coefficient in the exponents, we have that
\begin{alignat*}{1}
J_{ii} & =\Sigma_{ii}^{-1}+\sum_{j\in\calN(i)}\left(\left(\left[\begin{array}{cc}
\Sigma_{ii} & \Sigma_{ji}\\
\Sigma_{ij} & \Sigma_{jj}
\end{array}\right]^{-1}\right)_{11}-\Sigma_{ii}^{-1}\right)\\
 & =\left(1-\text{deg}(i)\right)\Sigma_{ii}^{-1}+\sum_{j\in\calN(i)}\left(\Sigma_{ii}-\Sigma_{ij}\Sigma_{jj}^{-1}\Sigma_{ji}\right)^{-1}
\end{alignat*}
and for $(i,j)\in\calE$, 
\begin{eqnarray*}
J_{ij} & = & \left(\left[\begin{array}{cc}
\Sigma_{ii} & \Sigma_{ij}\\
\Sigma_{ji} & \Sigma_{jj}
\end{array}\right]^{-1}\right)_{12}\\
 & = & \frac{\Sigma_{ij}}{\Sigma_{ij}^{2}-\Sigma_{ii}\Sigma_{jj}}
\end{eqnarray*}

The complexity of computing each $J_{ij}$, $(i,j)\in\calE$ is $O(1)$
and the complexity of computing each $J_{ii}$ is $O(\deg i)$. Since
$\Sigma_{i\in\calV}\deg(i)$ equals twice the number of edges, which
is $\calO(n)$, the total complexity is $\calO(n)$. 

\end{proof}

\begin{lemma}(The matrix inversion lemmas)

If $\left[\begin{array}{cc}
A & B\\
C & D
\end{array}\right]$ is invertible, we have 
\begin{eqnarray}
\left[\begin{array}{cc}
A & B\\
C & D
\end{array}\right]^{-1} & = & \left[\begin{array}{cc}
(A-BD^{-1}C)^{-1} & -(A-BD^{-1}C)^{-1}BD^{-1}\\
-D^{-1}C(A-BD^{-1}C)^{-1} & D^{-1}+D^{-1}C(A-BD^{-1}C)^{-1}BD^{-1}
\end{array}\right]\label{eq:inversion1}
\end{eqnarray}

or 
\begin{equation}
\left[\begin{array}{cc}
A & B\\
C & D
\end{array}\right]^{-1}=\left[\begin{array}{cc}
A^{-1}+A^{-1}B(D-CA^{-1}B)^{-1}CA^{-1} & -A^{-1}B(D-CA^{-1}B)^{-1}\\
-(D-CA^{-1}B)^{-1}CA^{-1} & (D-CA^{-1}B)^{-1}
\end{array}\right]\label{eq:inversion2-1}
\end{equation}

and
\begin{equation}
\left(A-BD^{-1}C\right)^{-1}=A^{-1}+A^{-1}B(D-CA^{-1}B)^{-1}CA^{-1}.\label{eq:inversion3}
\end{equation}

\label{lemma:inversion}

\end{lemma}

The proof of Lemma \ref{lemma:inversion} can be found in standard
matrix analysis books.

\subsection{Proof of Proposition \ref{prop:givenF}}

\begin{proof}

For a fixed FVS $F$, the LHS of \eqref{eq:KLfixedtree} is only a
function of the spanning tree among the non-feedback nodes. Hence,
the optimal set of edges among the non-feedback nodes can be obtained
by finding the maximum spanning tree of the subgraph induced by $T$
with $I_{\hat{p}}(\bx_{i};\bx_{j}|\bx_{F})\geq0$ being the edge weight
between $i$ and $j$. %
\footnote{In fact, we have given an algorithm to learn general models (not only
GGMs, but also other models, e.g., discrete ones) defined on graphs
with a given FVS $F$. However, we do not explore the general setting
in this paper.%
}

For Gaussian distributions, the covariance matrix of the distribution
$\hat{p}(\bx_{T}|\bx_{F})$ depends only on the set $F$ but is invariant
to the value of $\bx_{F}$. Hence, finding the optimal edge set of
the tree part is equivalent to running the Chow-Liu algorithm with
the input being the covariance matrix of $\hat{p}_{T|F}(\bx_{T}|\bx_{F})$,
which is simply $\hat{\Sigma}_{T|F}=\hat{\Sigma}_{T}-\hat{\Sigma}_{M}\hat{\Sigma}_{F}^{-1}\hat{\Sigma}_{M}^{T}$.
Let $\calE_{\text{CL}}=\text{CL}_{\calE}(\hat{\Sigma}_{T|F})$ and
$\Sigma_{\text{CL}}=\text{CL}(\hat{\Sigma}_{T|F})$. Denote the optimal
covariance matrix as $\Sigma_{\text{ML}}=\left[\begin{array}{cc}
\Sigma_{F}^{\text{ML}} & \left(\Sigma_{M}^{\text{ML}}\right)^{T}\\
\Sigma_{M}^{\text{ML}} & \Sigma_{T}^{\text{ML}}
\end{array}\right]$. According to \eqref{eq:optimalparameter-1}, we must have $\Sigma_{F}^{\text{ML}}=\hat{\Sigma}_{F}$
and $\Sigma_{M}^{\text{ML}}=\hat{\Sigma}_{M}$. From \eqref{eq:optimalparameter-1}
the corresponding conditional covariance matrix $\Sigma_{T|F}^{\text{ML}}$
of $\Sigma_{\text{ML}}$ must equal $\Sigma_{\text{CL}}$. Hence,
we have $\Sigma_{T|F}^{\text{ML}}=\Sigma_{T}^{\text{ML}}-\Sigma_{M}^{\text{ML}}\left(\Sigma_{F}^{\text{ML}}\right)^{-1}\left(\Sigma_{M}^{\text{ML}}\right)^{T}=\Sigma_{\text{CL}}.$
Therefore, we can obtain $\Sigma_{T}^{\text{ML}}=\text{CL}(\hat{\Sigma}_{T|F})+\hat{\Sigma}_{M}\hat{\Sigma}_{F}^{-1}\hat{\Sigma}_{M}^{T}$.
We also have that $\calE_{\text{ML}}=\calE_{\text{CL}}$ since $\calE_{\text{ML}}$
is defined to be the set of edges among the feedback nodes.

Now we analyze the complexity of Algorithm \ref{Algo:givenF}. The
matrix $\hat{\Sigma}_{T|F}$ is computed with complexity $\calO(kn^{2})$.
Computing the maximum weight spanning tree algorithm has complexity
${\cal O}(n^{2}\log n)$ using Kruskal's algorithm (the amortized
complexity can be further reduced, but it is not the focus of this
paper). Other operations have complexity $\calO(n^{2}).$ Hence, the
total complexity of Algorithm \ref{Algo:givenF} is ${\cal O}(kn^{2}+n^{2}\log n)$.

Next we proceed to prove that we can compute all the non-zero entries
of $J_{\text{ML}}=\left(\Sigma_{\text{ML}}\right)^{-1}$ in time $\calO(k^{2}n).$ 

Let $J_{\text{ML}}=$$\left[\begin{array}{cc}
J_{F}^{\text{ML}} & \left(J_{M}^{\text{ML}}\right)^{T}\\
J_{M}^{\text{ML}} & J_{T}^{\text{ML}}
\end{array}\right]$. We have that $J_{T}^{\text{ML}}=\left(\text{CL}(\hat{\Sigma}_{T|F})\right)^{-1}$
has tree structure with $\calT$. Therefore, the non-zero entries
of $J_{T}^{\text{ML}}$can be computed with complexity ${\cal O}(n-k)$.from
Lemma $5$.

From \eqref{eq:inversion2-1} we have 
\begin{alignat}{1}
J_{M}^{\text{ML}} & =-J_{T}^{\text{\text{ML}}}\Sigma_{M}^{\text{ML}}\left(\Sigma_{F}^{\text{ML}}\right)^{-1},\label{eq:J_M}
\end{alignat}
which can be computed with complexity ${\cal O}(k^{2}n)$ by matrix
multiplication in the regular order. Note that $J_{T}^{\text{\text{ML}}}\Sigma_{M}^{\text{ML}}$
is computed in $\calO(kn)$ since $J_{T}^{\text{ML}}$ only has $\calO(n)$
non-zero entries.

From \eqref{eq:inversion2-1} we have 
\[
J_{F}^{\text{ML}}=\left(\Sigma_{F}^{\text{ML}}\right)^{-1}\left(I+\left(\left(\Sigma_{M}^{\text{ML}}\right)^{T}J_{T}^{\text{ML}}\right)\left(\Sigma_{M}^{\text{ML}}\left(\Sigma_{F}^{\text{ML}}\right)\right)\right),
\]
which has complexity $\calO(k^{2}n)$ following the order specified
by the parentheses. Note that $\left(P_{M}^{\text{ML}}\right)^{T}J_{T}^{\text{ML}}$is
computed in $\calO(kn)$ because $J_{T}^{\text{ML}}$ only has $\calO(n)$
non-zero entries. Hence, we need extra complexity of $\calO(k^{2}n)$
to compute all the non-zero entries of $J_{\text{ML}}$. 

We have therefore completed the proof for Proposition \ref{prop:givenF}.

\end{proof}

For easy reference, we summarize the procedure to compute $J_{\text{ML}}$
in Algorithm \ref{algo:computeJ_FVS}.

\begin{algorithm}[H]
\begin{enumerate}
\item Compute $J_{T}^{\text{ML}}$ using \eqref{eq:invertTree} 
\item Compute $J_{M}^{\text{ML}}=-J_{T}^{\text{\text{ML}}}\Sigma_{M}^{\text{ML}}\Sigma_{F}^{-1}$
using sparse matrix multiplication
\item Compute $\left(\Sigma_{F}^{\text{ML}}\right)^{-1}\left(I+\left(\left(\Sigma_{M}^{\text{ML}}\right)^{T}J_{T}^{\text{ML}}\right)\left(\Sigma_{M}^{\text{ML}}\left(\Sigma_{F}^{\text{ML}}\right)\right)\right)$
following the order specified by the parentheses using sparse matrix
multiplication.
\end{enumerate}
\caption{Compute $J_{\text{ML}}=\left(\Sigma_{\text{ML}}\right)^{-1}$ after
running Algorithm \ref{Algo:givenF}}

\label{algo:computeJ_FVS}
\end{algorithm}

\section{Proof of Proposition \ref{prop:LatenChowLiu}}

\label{sec:Appen_Prop2}

In this section, we first prove a more general result stated in Lemma
\ref{lemma:projectionmonotonicity-1}.

\begin{lemma}

In Algorithm \ref{algo:alternateprojection-1}, if Step 2(a) and Step
2(b) can be computed exactly, then we have that $D(\hat{p}(\bx_{T})||q^{(t+1)}(\bx_{T}))\leq D(\hat{p}(\bx_{T})||q^{(t)}(\bx_{T}))$,
where the equality is satisfied if and only if $\hat{p}^{(t)}(\bx_{F},\bx_{T})=\hat{p}^{(t+1)}(\bx_{F},\bx_{T})$.

\label{lemma:projectionmonotonicity-1}

\end{lemma}

\begin{algorithm}[H]
\begin{enumerate}
\item Propose an initial distribution $q^{(0)}(\bx_{F},\bx_{T})\in\calQ_{F}$
\item Alternating projections:

\begin{enumerate}
\item \textbf{P1:} Project to the empirical distribution: 
\[
\hat{p}^{(t)}(\bx_{F},\bx_{T})=\hat{p}(\bx_{T})q^{(t)}(\bx_{F}|\bx_{T})
\]

\item \textbf{P2:} Project to the best fitting structure on all variables:
\[
q^{(t+1)}(\bx_{F},\bx_{T})=\arg\min_{q(\bx_{F},\bx_{T})\in\calQ_{F}}D(\hat{p}^{(t)}(\bx_{F},\bx_{T})||q(\bx_{F},\bx_{T}))
\]
.
\end{enumerate}
\end{enumerate}
\caption{Alternating Projection}

\label{algo:alternateprojection-1}
\end{algorithm}

\begin{proof}

For any $t$,

\begin{alignat}{1}
 & D(\hat{p}^{(t)}(\bx_{T},\bx_{F})||q^{(t)}(\bx_{F},\bx_{T}))\nonumber \\
= & \int_{\bx_{T},\bx_{F}}\hat{p}(\bx_{T})q^{(t)}(\bx_{F}|\bx_{T})\log\frac{\hat{p}(\bx_{T})q^{(t)}(\bx_{F}|\bx_{T})}{q^{(t)}(\bx_{F},\bx_{T})}\nonumber \\
= & \int_{\bx_{T},\bx_{F}}\hat{p}(\bx_{T})q^{(t)}(\bx_{F}|\bx_{T})\log\frac{\hat{p}(\bx_{T})}{q^{(t)}(\bx_{T})}\nonumber \\
= & \int_{\bx_{T}}\hat{p}(\bx_{T})\log\frac{\hat{p}(\bx_{T})}{q^{(t)}(\bx_{T})}\nonumber \\
= & D(\hat{p}^{(t)}(\bx_{T})||q^{(t)}(\bx_{T}))\label{eq:first_eq}
\end{alignat}

By the definition of $q^{(t+1)}$ in step (b), we have 
\begin{equation}
D(\hat{p}(\bx_{T},\bx_{F})||q^{(t+1)}(\bx_{F},\bx_{T}))\leq D(\hat{p}^{(t)}(\bx_{T},\bx_{F})||q^{(t)}(\bx_{F},\bx_{T})).\label{eq:first_ineq}
\end{equation}

Therefore,

\begin{alignat}{1}
 & D(\hat{p}(\bx_{T})||q^{(t)}(\bx_{T}))\nonumber \\
\stackrel{(a)}{=} & D(\hat{p}^{(t)}(\bx_{T},\bx_{F})||q^{(t)}(\bx_{F},\bx_{T}))\\
\stackrel{(b)}{\geq} & D(\hat{p}^{(t)}(\bx_{T},\bx_{F})||q^{(t+1)}(\bx_{F},\bx_{T}))\\
= & \int_{\bx_{T},\bx_{F}}\hat{p}(\bx_{T})q^{(t)}(\bx_{F}|\bx_{T})\log\frac{\hat{p}(\bx_{T})q^{(t)}(\bx_{F}|\bx_{T})}{q^{(t+1)}(\bx_{F},\bx_{T})}\nonumber \\
= & \int_{\bx_{T},\bx_{F}}\hat{p}(\bx_{T})q^{(t)}(\bx_{F}|\bx_{T})\log\frac{\hat{p}(\bx_{T})}{q^{(t+1)}(\bx_{T})}+\int_{\bx_{T},\bx_{F}}\hat{p}(\bx_{T})q^{(t)}(\bx_{F}|\bx_{T})\log\frac{q^{(t)}(\bx_{F}|\bx_{T})}{q^{(t+1)}(\bx_{F}|\bx_{T})}\nonumber \\
= & \int_{\bx_{T}}\hat{p}(\bx_{T})\log\frac{\hat{p}(\bx_{T})}{q^{(t+1)}(\bx_{T})}+\int_{\bx_{T},\bx_{F}}\hat{p}(\bx_{T})q^{(t)}(\bx_{F}|\bx_{T})\log\frac{q^{(t)}(\bx_{F}|\bx_{T})\hat{p}(\bx_{T})}{q^{(t+1)}(\bx_{F}|\bx_{T})\hat{p}(\bx_{T})}\\
= & D(\hat{p}(\bx_{T})||q^{(t+1)}(\bx_{T}))+\int_{\bx_{T},\bx_{F}}\hat{p}^{(t)}(\bx_{F},\bx_{T})\log\frac{\hat{p}^{(t)}(\bx_{F},\bx_{T})}{\hat{p}^{(t+1)}(\bx_{F},\bx_{T})}\nonumber \\
= & D(\hat{p}(\bx_{T})||q^{(t+1)}(\bx_{T}))+D(\hat{p}^{(t)}(\bx_{F},\bx_{T})||\hat{p}^{(t+1)}(\bx_{F},\bx_{T}))\nonumber \\
\stackrel{(c)}{\geq} & D(\hat{p}(\bx_{T})||q^{(t+1)}(\bx_{T})),
\end{alignat}
where (a) is due to \eqref{eq:first_eq}, (b) is due to \eqref{eq:first_ineq},
and (c) is due to that $D(\hat{p}^{(t)}(\bx_{F},\bx_{T})||\hat{p}^{(t+1)}(\bx_{F},\bx_{T}))\geq0$.
Therefore, we always have $D(\hat{p}(\bx_{T})||q^{(t)})\geq D(\hat{p}(\bx_{T})||q^{(t+1)})$.
A necessary condition for the objective function to remain the same
is that $D(\hat{p}^{(t)}(\bx_{F},\bx_{T})||\hat{p}^{(t+1)}(\bx_{F},\bx_{T}))=0$,
which implies that $\hat{p}^{(t)}(\bx_{F},\bx_{T})=\hat{p}^{(t+1)}(\bx_{F},\bx_{F})$.
Hence, it further implies that $q^{(t)}(\bx_{F},\bx_{T})=q^{(t+1)}(\bx_{F},\bx_{T})$
under non-degenerate cases. Therefore, $\hat{p}^{(t)}(\bx_{F},\bx_{T})=\hat{p}^{(t+1)}(\bx_{F},\bx_{F})$
is a necessary and sufficient condition for the objective function
to remain the same. This completes the proof for Lemma \ref{lemma:projectionmonotonicity-1}.

\end{proof}

Now we proceed to the proof for Proposition \ref{prop:LatenChowLiu}.

\begin{proof}

Use the same notation as in the latent Chow-Liu algorithm (Algorithm
\ref{algo:GaussEM}). Let $\hat{p}(\bx_{T})=\calN(\bzero,\hat{\Sigma}_{T})$,
$p^{(t)}(\bx_{F},\bx_{T})=\calN(\bzero,\Sigma^{(t)}).$ Then 
\begin{alignat*}{1}
\hat{p}(\bx_{T}) & =\frac{1}{\sqrt{\det\left(2\pi\hat{\Sigma}_{T}\right)}}\exp\{-\frac{1}{2}\bx_{T}^{T}\hat{\Sigma}_{T}^{-1}\bx_{T}\}\\
p^{(t)}(\bx_{F}|\bx_{T}) & =\frac{1}{\sqrt{\det\left(2\pi\left(J_{F}^{^{(t)}}\right)^{-1}\right)}}\exp\{-\frac{1}{2}\left(\bx_{F}-\left(J_{F}^{(t)}\right)^{-1}J_{M}^{(t)}\bx_{T}\right)^{T}J_{F}^{(t)}\left(\bx_{F}-\left(J_{F}^{(t)}\right)^{-1}J_{M}^{(t)}\bx_{T}\right)^{T}\}
\end{alignat*}

Hence, following Algorithm \ref{algo:alternateprojection-1}, we have
\begin{alignat*}{1}
\hat{p}^{(t)}(\bx_{F},\bx_{T}) & =\hat{p}(\bx_{T})q^{(t)}(\bx_{F}|\bx_{T})\\
 & \propto\exp\{-\frac{1}{2}\left[\begin{array}{c}
\bx_{F}\\
\bx_{T}
\end{array}\right]^{T}\left[\begin{array}{cc}
J_{F}^{(t)} & \left(J_{M}^{(t)}\right)^{T}\\
J_{M}^{(t)} & \hat{\Sigma}_{T}^{-1}+J_{M}^{(t)}(J_{F}^{(t)})^{-1}(J_{M}^{(t)})^{T}
\end{array}\right]\left[\begin{array}{c}
\bx_{F}\\
\bx_{T}
\end{array}\right]\},
\end{alignat*}
which gives the same expression as in \textbf{P1} of Algorithm \ref{algo:GaussEM}.
The next projection 
\[
q^{(t+1)}(\bx_{F},\bx_{T})=\min_{q(\bx_{F},\bx_{T})\in\calQ_{F}}D(\hat{p}^{(t)}(\bx_{F},\bx_{T})||q(\bx_{F},\bx_{T}))
\]
has same form as M-L learning problem in Section \ref{sub:knownFVS},
and therefore can be computed exactly using the conditioned Chow-Liu
algorithm (Algorithm \ref{Algo:givenF}). By Lemma \ref{lemma:projectionmonotonicity-1},
we have thus proved the first part of Proposition \ref{prop:LatenChowLiu}.
The second part about the complexity of an accelerated version is
proved in Section \ref{sec:acceleratedLatent}.

\end{proof}

\section{The Accelerated Latent Chow-Liu Algorithm}

\label{sec:acceleratedLatent}

In this section, we describe the accelerated latent Chow-Liu algorithm
(Algorithm \ref{algo: accelaratedLatentChowLiu}), which gives exactly
the same result as the latent Chow-Liu algorithm \ref{algo:GaussEM},
but has a lower complexity of $\calO(kn^{2}+n^{2}\log n)$ per iteration.
The main complexity reduction is due to the use of Algorithm \ref{algo:computeJ_FVS}. 

We will use the following lemma in the proof of Algorithm \ref{algo: accelaratedLatentChowLiu}.

Now we proceed to prove the correctness of the accelerated Chow-Liu
algorithm and obtain its complexity.

\begin{proof}

In \textbf{P1} of the latent Chow-Liu algorithm (Algorithm \ref{algo:GaussEM})
we have 
\[
\hat{J}^{(t)}=\left[\begin{array}{cc}
J_{F}^{(t)} & (J_{M}^{(t)})^{T}\\
J_{M}^{(t)} & \left(\hat{\Sigma}_{T}\right)^{-1}+J_{M}^{(t)}(J_{F}^{(t)})^{-1}(J_{M}^{(t)})^{T}
\end{array}\right].
\]
Without explicitly computing $\hat{J}^{(t)}$, we can directly compute
$\hat{\Sigma}^{(t)}=\left(\hat{J}^{(t)}\right)^{-1}$ as follows. 

Let $A=J_{F}^{(t)}$, $B=(J_{M}^{(t)})^{T}$, $C=\hat{J}_{M}^{(t)}$,
and $D=\left(\hat{\Sigma}_{T}\right)^{-1}+J_{M}^{(t)}(J_{F}^{(t)})^{-1}(J_{M}^{(t)})^{T}$).
From \eqref{eq:inversion2-1} we have 
\[
\hat{\Sigma}_{F}^{(t)}=\left(J_{F}^{(t)}\right)^{-1}+\left(J_{F}^{(t)}\right)^{-1}\left(J_{M}^{(t)}\right)^{T}(D-CA^{-1}B)^{-1}\hat{J}_{M}^{(t)}\left(J_{F}^{(t)}\right)^{-1}
\]
and 
\begin{equation}
\hat{\Sigma}_{T}^{(t)}=(D-CA^{-1}B)^{-1}=\hat{\Sigma}_{T}.\label{eq:Sigma_T}
\end{equation}
 
\begin{equation}
\hat{\Sigma}_{F}^{(t)}=\left(J_{F}^{(t)}\right)^{-1}+\left(J_{F}^{(t)}\right)^{-1}\left(J_{M}^{(t)}\right)^{T}\hat{\Sigma}_{T}\hat{J}_{M}^{(t)}\left(J_{F}^{(t)}\right)^{-1}.\label{eq:Sigma_F}
\end{equation}

Also from \eqref{eq:inversion2-1}, we have that 
\begin{equation}
\hat{\Sigma}_{M}^{(t)}=-\hat{\Sigma}_{T}J_{M}^{(t)}\left(J_{F}^{(t)}\right)^{-1}.\label{eq:Sigma_M}
\end{equation}

It can be checked that the matrix multiplications of \eqref{eq:Sigma_T},
\eqref{eq:Sigma_F}, and \eqref{eq:Sigma_M} have complexity $\calO(kn^{2})$.

\textbf{P2} in Algorithm \ref{algo:GaussEM} can be computed with
complexity ${\cal O}(n^{2}k+n^{2}\log n)$ from Proposition \ref{prop:givenF}.
Therefore, the complexity of this accelerated version (summarized
in Algorithm \ref{algo: accelaratedLatentChowLiu}) is $\calO(n^{2}k+n^{2}\log n)$
per iteration. We have thus completed the proof for Proposition \ref{prop:LatenChowLiu}.

\begin{algorithm}[H]
\textbf{Input:} the empirical covariance matrix $\hat{\Sigma}_{T}$

\textbf{Output:} information matrix $J=\left[\begin{array}{cc}
J_{\calF} & J_{M}^{T}\\
J_{M} & J_{\calT}
\end{array}\right]$.
\begin{enumerate}
\item \noindent Initialization: $J^{(0)}=\left[\begin{array}{cc}
J_{F}^{(0)} & \left(J_{M}^{(0)}\right)^{T}\\
J_{M}^{(0)} & J_{T}^{(0)}
\end{array}\right]$.
\item Repeat

\begin{enumerate}
\item \textbf{AP1:} Compute 
\begin{alignat*}{1}
\hat{\Sigma}_{F}^{(t)} & =\left(J_{F}^{(t)}\right)^{-1}+\left(Y^{(t)}\right)^{T}\hat{\Sigma}_{T}Y^{(t)}\\
\hat{\Sigma}_{T}^{(t)} & =\hat{\Sigma}{}_{T}\\
\hat{\Sigma}{}_{M}^{(t)} & =-\hat{\Sigma}_{T}Y^{(t)},
\end{alignat*}
where $Y^{(t)}=\hat{J}_{M}^{(t)}\left(J_{F}^{(t)}\right)^{-1}$\\
Let $\hat{\Sigma}^{(t)}=\left[\begin{array}{cc}
\hat{\Sigma}_{F}^{(t)} & \left(\hat{\Sigma}_{M}^{(t)}\right)^{T}\\
\hat{\Sigma}_{M}^{(t)} & \hat{\Sigma}_{T}
\end{array}\right]$.
\item \textbf{AP2:} Compute $\Sigma^{(t+1)}$ and $J^{(t+1)}$=$\left(\Sigma^{(t+1)}\right)^{-1}$
from $\hat{\Sigma}^{(t)}$ using Algorithm \ref{Algo:givenF} and
Algorithm \ref{algo:computeJ_FVS}: 
\begin{alignat*}{1}
J^{(t+1)} & =\left[\begin{array}{cc}
J_{F}^{(t+1)} & \left(J_{M}^{(t+1)}\right)^{T}\\
J_{M}^{(t+1)} & J_{T}^{(t+1)}
\end{array}\right]\\
\Sigma^{(t+1)} & =\left[\begin{array}{cc}
\Sigma_{F}^{(t+1)} & \left(\Sigma_{M}^{(t+1)}\right)^{T}\\
\Sigma_{M}^{(t+1)} & \Sigma_{T}^{(t+1)}
\end{array}\right]
\end{alignat*}

\end{enumerate}
\end{enumerate}
\caption{The accelerated Chow-Liu algorithm}

\label{algo: accelaratedLatentChowLiu}
\end{algorithm}

\end{proof}

\end{document}